\documentclass[nohyperref]{article}

\usepackage{microtype}
\usepackage{graphicx}
\usepackage{booktabs} 
\usepackage{caption}
\usepackage{subcaption}
\usepackage{bbm}

\usepackage{hyperref}

\usepackage[accepted]{icml2023}

\usepackage{amsmath}
\usepackage{amssymb}
\usepackage{mathtools}
\usepackage{amsthm}

\usepackage[capitalize,noabbrev]{cleveref}

\theoremstyle{plain}

\theoremstyle{definition}
\newtheorem{definition}{Definition}
\newtheorem{assumption}{Assumption}
\theoremstyle{remark}

\usepackage[textsize=tiny]{todonotes}

\icmltitlerunning{Generative Adversarial Symmetry Discovery}

\usepackage{pifont}
\usepackage{caption}
\usepackage{subcaption}
\usepackage[framemethod=tikz]{mdframed}
\usepackage{mathtools}
\usepackage{cuted}
\usepackage{xcolor}
\usepackage{hhline}
\usepackage{thmtools, thm-restate}
\usepackage{array}

\newcommand{\model}{{LieGAN}}
\newcommand{\cmark}{\ding{51}}%
\newcommand{\xmark}{\ding{55}}%

\DeclareMathOperator*{\argmin}{arg\,min}

\newcommand{\T}[1]{{\mathcal{#1}}} 

\begin{document}

\twocolumn[
\icmltitle{Generative Adversarial Symmetry Discovery}



\icmlsetsymbol{equal}{*}

\begin{icmlauthorlist}
\icmlauthor{Jianke Yang}{ucsd}
\icmlauthor{Robin Walters}{equal,ne}
\icmlauthor{Nima Dehmamy}{equal,ibm}
\icmlauthor{Rose Yu}{ucsd}
\end{icmlauthorlist}

\icmlaffiliation{ucsd}{University of California San Diego}
\icmlaffiliation{ne}{Northeastern University}
\icmlaffiliation{ibm}{IBM Research}

\icmlcorrespondingauthor{Rose Yu}{roseyu@ucsd.edu}

\icmlkeywords{symmetry discovery, equivariance, equivariant neural network, geometric deep learning, scientific machine learning, Lie theory, generative adversarial training}

\vskip 0.3in
]



\printAffiliationsAndNotice{\icmlEqualContribution} 

\begin{abstract}
Despite the success of equivariant neural networks in scientific applications, they require knowing the symmetry group a priori. 
However, it may be difficult to know which symmetry to use as an inductive bias in practice. Enforcing the wrong symmetry could even hurt the performance.
In this paper, we propose a framework, \model, to \textit{automatically discover equivariances} from a dataset using a paradigm akin to generative adversarial training.  Specifically,  a generator learns a group of transformations applied to the data, which preserve the original distribution and fool the discriminator. \model~represents symmetry as interpretable Lie algebra basis and can discover various symmetries such as the rotation group $\mathrm{SO}(n)$,  restricted Lorentz group $\mathrm{SO}(1,3)^+$ in trajectory prediction and top-quark tagging tasks. The learned symmetry can also be readily used in several existing equivariant neural networks to improve accuracy and generalization in prediction. Our code is available at \href{https://github.com/Rose-STL-Lab/LieGAN}{https://github.com/Rose-STL-Lab/LieGAN}.
\end{abstract}

\section{Introduction}

Symmetry is an important inductive bias in deep learning. For example, convolutional neural networks \cite{cnn} exploit translational symmetry in images, and graph neural networks utilize permutation symmetry in graph-structured data \cite{gcn}. Equivariant networks have led to significant improvement in generalization, sample efficiency and scientific validity \cite{deepsets, e2cnn, gaugecnn, wang2021incorporating}. Interest has surged in  both theoretical analysis and practical techniques for building general group equivariant neural networks \cite{kondor2018generalization, cohen2019general, bspline-cnn, emlp}. 

\vspace{1mm}

\begin{figure}[h]
    \centering
    \includegraphics[width=0.46\textwidth]{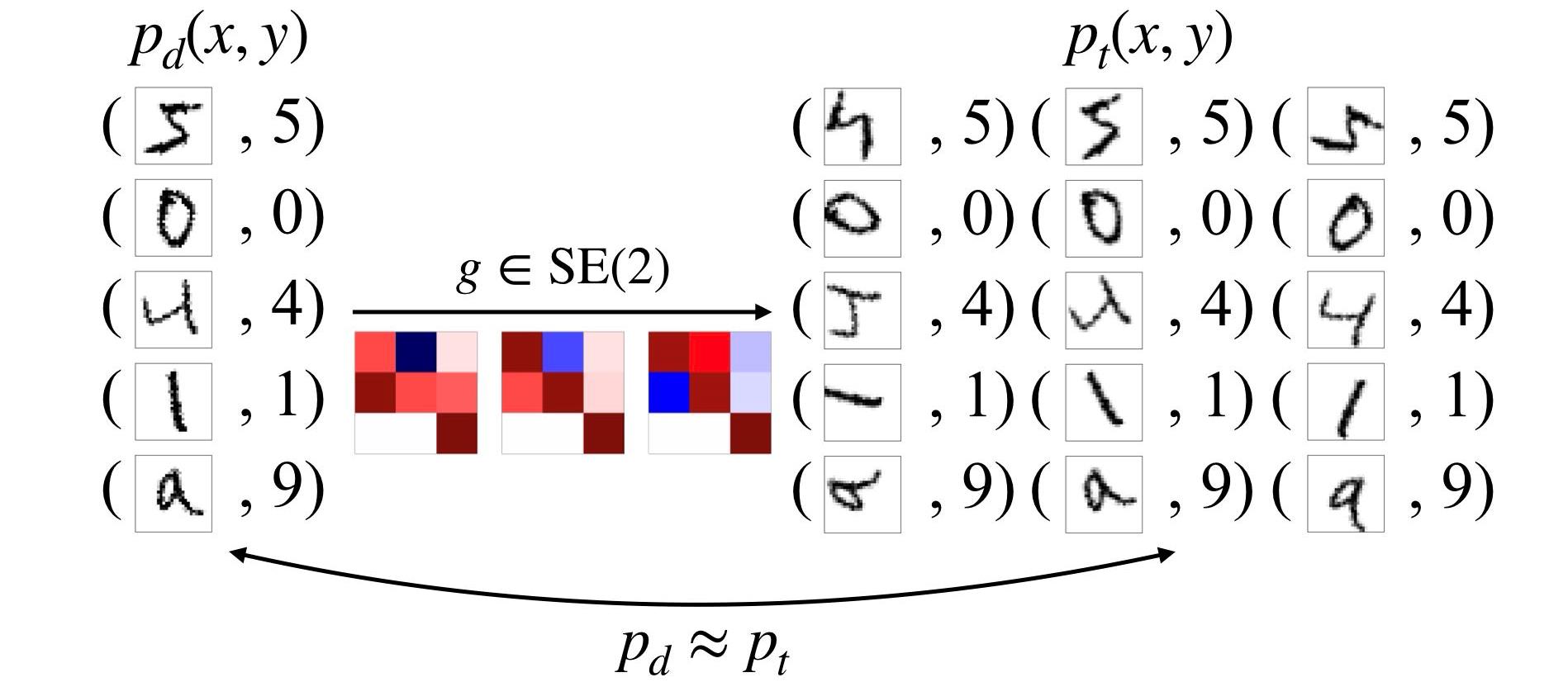}
    \caption{Connection between symmetry and data distribution. MNIST classification is invariant to a subset of $\mathrm{SE}(2)$ transformations. If the digits are transformed by random rotations and translations within this subset, the resulting data distribution $p_t$ remains close to the original distribution $p_d$.}
    \label{fig:mnist-demo}
    \vspace{-2mm} 
\end{figure}

However, a key limitation of  equivariant neural networks is that they require explicit knowledge of the data symmetry before a model can be constructed. In practice, it is sometimes difficult to identify the true symmetries of the data, and constraining the model by the exact mathematical symmetry might not be optimal in real-world situations \cite{ruiapprox}. These challenges call for approaches that can  automatically discover the underlying symmetry of the data. 

Neural networks that discover unknown symmetry play the role of AI scientists
, not only by making data-driven predictions, but also by identifying and describing physical systems through their symmetries and generating new scientific insights through the close relationship between symmetry, conservation laws and underlying governing equations \cite{noethernetwork}.
Most existing work on symmetry discovery can only address a small fraction of potential symmetry types, such as finite groups \cite{msr}, subsets of a given group \cite{augerino} or individual group elements \cite{sgan}. L-conv \cite{lconv} can discover continuous symmetries without discretization of the groups, but has limited computational efficiency. A more general framework is needed for the discovery of various real-world symmetries.


In this work, we present a novel framework for discovering continuous symmetry from data using  generative adversarial training \cite{gan}.
We establish the connection between symmetry and data distribution as in figure \ref{fig:mnist-demo}. Our method then trains a symmetry generator that transforms the training data and outputs a similar distribution to the original dataset, which suggests equivariance or invariance to the learned tranformations. Making use of the theory of Lie groups and Lie algebras, our method, \textit\model{}, is able to discover continuous symmetries as matrix groups. Moreover, through different parameterization strategies, it can also deal with other types of symmetries, such as discrete group transformation, as well as the subset of a group.

Our main contributions can be summarized as follows: 
\begin{enumerate}
    \item We propose \model, a method for automatically discovering symmetries from data, capable of learning general linear symmetries, including the rotation group $\mathrm{SO}(n)$ and restricted Lorentz group $\mathrm{SO}(1,3)^+$.
    \item \model~is interpretable, directly yielding an orthogonal Lie algebra basis as a discovery result. 
    \item We show that the Lie algebra learned by \model~leads to excellent performance in downstream tasks such as $N$-body dynamics and top quark tagging.
    \item We propose LieGNN, a modified E(n) Equivariant Graph Neural Network (EGNN) \cite{egnn} that integrates symmetries learned by \model, achieving similar performance to equivariant models with ground truth symmetries. 
\end{enumerate}

\section{Related Work}
\paragraph{Equivariant Neural Networks.}
Many works have addressed the problems of designing neural network modules that are equivariant to specific transformations, such as permutation in sets \cite{deepsets}, local gauge transformations \cite{gaugecnn}, scaling \cite{deepscale}, rotation on spheres \cite{spherecnn} and general $\mathrm E(2)$ transformations on Euclidean plane \cite{e2cnn}. Another branch of works focus on developing theoretical guidelines and practical methods for building general group equivariant neural networks \cite{gconv, kondor2018generalization, cohen2019general, lieconv, emlp}. However, these methods rely on explicit a priori knowledge of the data symmetry. Instead, we are interested in discovering knowledge of the symmetry itself. The learned symmetry can then be used to select or design an equivariant neural network to make predictions.

\paragraph{Generative Adversarial Training.} The original generative adversarial network (GAN) \cite{gan} uses a generator to transform random noise into target distribution. Many variants of GAN have been proposed to address different tasks other than the unrestricted generation \cite{cgan, stylegan, Isola2016ImagetoImageTW, cyclegan, dagan}. In particular, CycleGAN \cite{cyclegan} learns a generator that maps the input image to another domain. DAGAN \cite{dagan} also takes data points from a source domain and generalizes them to a broader domain with a generator to perform data augmentation, which is related to our task, as the augmenting process can be regarded as a set of transformations to which the data is invariant. These works use samples from the original distribution instead of random noise as generator input and perform domain transfer or generalization with a generator. Our work proposes another usage of such design.
The generator in our model produces transformations that are applied to data samples, and discovers the underlying symmetry by learning the correct set of transformations.

\paragraph{Symmetry Discovery.}
Many existing symmetry discovery methods \cite{augerino,msr,p-gcnn,krippendorf2020detecting} limit their search space to a small fraction of potential symmetry types. MSR \cite{msr} reparameterizes network weights into task weights and a symmetry matrix and meta-learns the symmetry matrix to provide information about task symmetry. However, it can only be applied to finite groups and scales linearly in space complexity with the size of the group, which eliminates the possibility of applying this algorithm to infinite continuous groups. Augerino \cite{augerino} addresses a different but relevant scenario: learning the extent of symmetry within a given group. Partial G-CNN \cite{p-gcnn} also learns group subsets via distributions on group to describe the symmetry at different levels in the model. These approaches can only be applied to cases where the symmetry group is known. \citet{krippendorf2020detecting} proposes to detect symmetries by constructing a synthetic classification task and examining the structure of network embedding layers. This method involves some manual procedures, such as defining the classification task and choosing the metric for latent space analysis. Our work aims to address all of the above limitations within a unified, automated framework.

The theory of Lie group and Lie algebra plays an important role in describing continuous symmetries. LieGG \cite{liegg} extracts symmetry learned by neural network by solving for network polarization matrix and proposes several metrics to evaluate the degree of symmetry. L-conv \cite{lconv} develops a Lie algebra convolutional network that can model any group equivariant functions. However, it performs first order approximation for matrix exponential and uses recursive layers to push the kernel away from the identity, which may become too expensive in practice.

\citet{sgan} also proposes to discover symmetries of the dataset distribution with a GAN. A major limitation of their algorithm is that the model can only learn one group element in a single round of training and has to rely on other techniques such as subgroup regularization or group composition to identify the group. Also, their definition of symmetry is different from ours.

Comparison between \model~and other works on symmetry discovery can be found in Table \ref{tab:comparison}. To the best of our knowledge, our approach is the \textit{first} to address the discovery of such a variety of symmetries including discrete group, continuous group, and subset of given or unknown group.

\begin{table}[h]
\caption{Comparison of different models' capability of discovering different kinds of symmetries}
\begin{center}
\begin{small}
\begin{sc}
    \begin{tabular}{c|>{\centering\arraybackslash}p{.4cm}>{\centering\arraybackslash}p{1.2cm}>{\centering\arraybackslash}p{1cm}>{\centering\arraybackslash}p{1cm}}
    \hline
        Symmetry & MSR & Augerino & LieGG& \model \\
        \hline
        Discrete & \cmark & \xmark & \xmark & \cmark \\
        Continuous & \xmark & \xmark & \cmark & \cmark \\
        Group subset & \xmark & \cmark & \xmark & \cmark \\
    \hline
    \end{tabular}
\end{sc}
\end{small}
\end{center}
    
    \label{tab:comparison}
\end{table}

\section{Background}
Before presenting our methodology, we provide some preliminary concepts that appear frequently in our work. We assume basic knowledge about group theory.
\paragraph{Lie group.} A Lie group is a group that is also a differentiable manifold. It can be used to describe continuous transformations. For example, all 2D rotations form a Lie group $\mathrm{SO}(2)$, where rotation with angle $\theta$ can be represented by $R=\begin{bmatrix}\cos\theta & -\sin\theta \\ \sin\theta & \cos\theta\end{bmatrix}$. Also, all Euclidean transformations including reflection, rotation and translation form the Lie group of $\mathrm{E}(n)$. Each Lie group is associated with a Lie algebra, which is its tangent vector space at identity: $\mathfrak g=T_{\mathrm{Id}}G$. The basis of the Lie algebra $L_i\in\mathfrak g$ are called (infinitesimal) generators of the Lie group. Group elements that are infinitesimally close to identity can be written in terms of these generators: $g=\mathrm{Id}+\sum_i\epsilon_iL_i$.

We can use an exponential map $\exp:\mathfrak g\rightarrow G$ to map Lie algebra elements to Lie group elements. For matrix groups, matrix exponential is such a map. For a connected Lie group $G$, its elements can be written as $g=\exp(\sum_iw_iL_i)$. 

\paragraph{Group representation.}
We are interested in how group elements transform the data. 
We assume that the input space is $\mathcal X=\mathbb R^n$. 
A group element $g \in G$ can act linearly on $x\in \mathcal{X}$ via $\rho_\mathcal{X}(g)$, where $\rho_\mathcal{X}:G\rightarrow \mathrm{GL}(n)$ is a group representation.
$\rho_\mathcal{X}$ maps each group element $g$ to a nonsingular matrix $\rho_\mathcal{X}(g)\in\mathbb R^{n\times n}$ that transforms the input vector.

A group representation $\rho:G\to \mathrm{GL}(n)$ induces a representation for the Lie algebra $\mathfrak{g}=T_{\mathrm{Id}}G$ denoted as $d\rho:\mathfrak g\rightarrow \mathfrak{gl}(n)$, which relate to the representation of its Lie group by $\exp(d\rho(L))=\rho(\exp(L))$.

\section{Symmetry Discovery}
We aim to automatically discover symmetry from data. 
Formally, let $\mathcal{D}=\{({x}_i, {y}_i) \}_{i=1}^N$ be a dataset with distribution $ x_i, y_i\sim p_d( x, y)$, input space $\T{X}=\mathbb{R}^n$, output space $\T{Y}=\mathbb{R}^m$ and an unknown function $f:\T{X}\rightarrow\T{Y}$. 
We have: 
\begin{definition}[Equivariance]
Suppose a group $G$ acts on $\mathcal{X}$ and $\mathcal{Y}$ via representations $\rho_\mathcal X: G\rightarrow \mathrm{GL}(n)$ and $\rho_\mathcal Y: G\rightarrow \mathrm{GL}(m)$. 
Then, a function $f: \T{X}\to \T{Y} $ is \emph{equivariant} if $\forall g\in G$,  $( x, y)\in \mathcal{D}$, $\rho_\mathcal Y(g) y=f(\rho_\mathcal X(g) x)$.
 We omit $\rho_\mathcal X$ and $\rho_\mathcal Y$ when clear and write $g y=f(g x)$.
 \label{thm:sym_def}
\end{definition}
%

We also address invariance in this work, which is a special case of equivariance when $\rho_\T{Y}(g)=\mathrm{Id}$. Next, we describe the formulation to relate symmetry discovery with generative adversarial training \cite{gan}.  

\subsection{Generative Adversarial Symmetry Discovery} 
By Definition \ref{thm:sym_def}, if a group element acts on the input, the output of an equivariant function is also transformed correspondingly by the representation of the same element. From another perspective, if all the data samples are transformed in this way, the transformed data distribution should remain similar to the original dataset distribution, as is demonstrated in figure \ref{fig:mnist-demo}.

At a high level,  we want to design a generator that can efficiently produce transformed input  and a discriminator that can not distinguish real samples from the dataset and the outputs from the generator. Through adversarial training, the generator tries to fool the discriminator by learning a group of transformations that minimize the divergence between the transformed and the original distributions. This group of transformations defines the symmetry  of interest.

We present our symmetry discovery framework in figure \ref{fig:structure}. We are interested in how the group $G$ acts on data through its representations $\rho_\mathcal X$ and $\rho_\mathcal Y$.  We learn $G$ as a subgroup of $\mathrm{GL}(k)$ for some $k$ chosen based on the task. The representations $\rho_\mathcal X \colon \mathrm{GL}(k) \to \mathrm{GL}(n)$ and $\rho_\mathcal Y \colon \mathrm{GL}(k) \to \mathrm{GL}(m)$ are also chosen and fixed based on the task. 
The GAN generator $\Phi$ samples an element from a distribution $\mu$ defined on $\mathrm{GL}(k)$ and then applies it to $ x$ and $ y$:
%
\begin{equation}
    \Phi( x, y)=(\rho_\mathcal{X}(g) x,\rho_\mathcal{Y}(g) y)
\end{equation}
For an invariant task, for example, we set $k = n$ and $\rho_\mathcal{X} = \mathrm{Id}$ the standard representation and $\rho_\mathcal{Y} = 1$ the trivial representation. For a time series prediction task, predicting a system state based on $t$ previous states, we set $k = m$ and $n = tm$ and $\rho_\mathcal{X} = \mathrm{Id}^{\oplus t}$ and $\rho_\mathcal{Y} = \mathrm{Id}$.  

\begin{figure*}[t!]
    \centering
    \includegraphics[width=.9\textwidth]{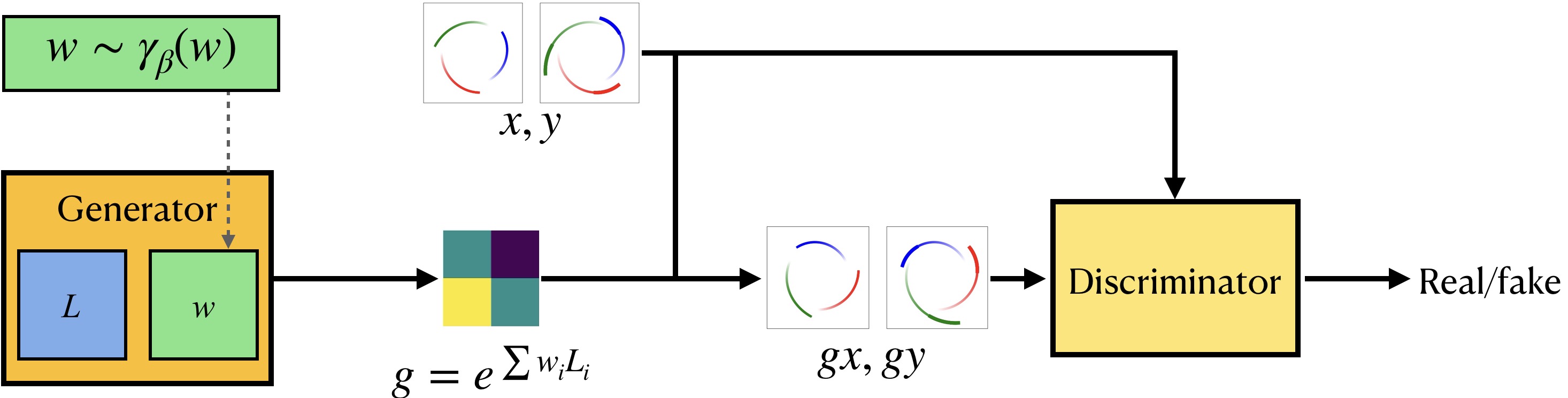}
    \caption{
    Structure of the proposed \model~model. The transformation generator learns a continuous Lie group acting on the data that preserves the original joint distribution. For example, this figure shows a task of predicting future 3-body movement based on past observations, where the generator could learn rotation symmetry.}
    \label{fig:structure}
\end{figure*}

Under this formulation, the generator should learn a subgroup of the transformations to which $f:\mathcal X \rightarrow\mathcal Y$ is equivariant. That is, it should generate a distribution close to the original data distribution. Similar to the setting in GAN, we optimize the following minimax objective:
\begin{align}
    &\min_\Phi\max_D 
    L(\Phi,D)\cr =&\mathbb{E}_{ x, y\sim p_d,g\sim\mu}\Big[\log D( x, y) 
    +\log(1- D(\Phi( x, y)))\Big]\cr
    =&\mathbb{E}_{ x, y\sim p_d}[\log D( x, y)]+\mathbb{E}_{ x, y\sim p_g}[\log(1- D( x, y))] \label{eq:gan_loss}
\end{align}
where $D$ is a standard GAN discriminator that outputs a real value as the probability that $(x,y)$ is a real sample, $p_d$ is the density of the original data distribution and $p_g$ is the generator-transformed distribution given according to change-of-variable formula by
\begin{equation}
    p_g( x, y)
=\int_g\mu(g)p_{d}(g^{-1} x,g^{-1} y)/(|\rho_\mathcal{X}(g)||\rho_\mathcal{Y}(g)|)dg
    \label{eq:p_gen}
\end{equation}

Under the ideal discriminator, the generator in the original GAN formulation minimizes the JS divergence between two distributions \cite{fgan}. In our setting, we prove that our generator can achieve zero divergences with the correct symmetry group under certain circumstances.

\begin{restatable}{theorem}{strong}
The generator can achieve zero JS divergence by learning a maximal subgroup $G^* \subset \mathrm{GL}(n)$ with respect to which $y=f(x)$ is equivariant if $p_d(x)$ is distributed proportionally to the volume of inverse group element transformation along each orbit of $G^*$-action on $\mathcal X$, that is, $p_d(gx_0)\propto |\rho_\mathcal X (g^{-1})||\rho_\mathcal Y (g^{-1})|$.
\label{thm:strong}
\end{restatable}

The hypothesis of Theorem \ref{thm:strong} is equivalent to saying that $p_d(x)$ is uniform along each group action orbit when the transformation is volume preserving, as in the case of rotation. However, as this is often not satisfied in practice, there is no guarantee that the generator can achieve zero divergences with nonidentical transformations. Nevertheless, as formalized in the following theorem, the generator can learn a nontrivial symmetry under some weak assumptions.

\begin{restatable}{theorem}{weak}
Under assumptions \ref{assum:existence}, \ref{assum:neg} and \ref{assum:uniformity}, the GAN loss function under the ideal discriminator $L(\Phi,{D^*})$ is lower with a generator that learns a subspace of the true Lie algebra $\mathfrak{g}^*$ than a generator with an orthogonal Lie algebra to $\mathfrak{g}^*$. That is, if $\mathfrak{g}_1\cap\mathfrak{g}^*\not={\mathrm{\{\mathbf0\}}}$, $\mathfrak{g}_2\cap\mathfrak{g}^*=\mathrm{\{\mathbf0\}}$, then $L(\mathfrak{g}_1, D^*)<L(\mathfrak{g}_2,D^*)=0$.
\label{thm:weak}
\end{restatable}

Theorem \ref{thm:weak} ensures that a partially correct symmetry results in lower loss function value than an incorrect symmetry. In other words, optimizing \eqref{eq:gan_loss} leads to symmetry discovery.

The related assumptions and proofs for Theorem \ref{thm:strong} and \ref{thm:weak} are deferred to Appendix \ref{prf:gan}.

\subsection{Parameterizing Distributions Over Lie Group}
We use the theory of Lie groups to model continuous sets of transformations.
To parameterize a distribution on a Lie group with $c$ dimensions and $k$ representation dimensions, our model learns Lie algebra generators $\{L_i\in\mathbb R^{k\times k}\}_{i=1}^c$ and samples the coefficients $w_i\in\mathbb R$ for their linear combination from either a fixed or a learnable distribution. The Lie algebra element is then mapped to a Lie group element using the matrix exponential \cite{reparam-lie}.
\begin{equation}
    w\sim \gamma_\beta(w), \quad
    g=\exp\Big[\sum_iw_iL_i\Big]
    \label{eq:lie-parameterization}
\end{equation}
The coefficient distribution $\gamma_\beta$ can be either fixed or updated, depending on our focus of discovery. 
If we have little information on the group, then by learning the $L_i$ and leaving the coefficient distribution fixed, our model can still express distributions over many different groups. On the other hand, we may want to find a subgroup or a subset of some known group. For example, the symmetry may be some discrete subgroup of SO(2) for some tasks. In this case, we fix $L$ as the rotation generator and learn $\gamma_\beta$, revealing peaks at certain values. Learning $\gamma_\beta$ is also useful when the task is not equivariant to the full group, but displays invariance for a subset of transformations, for example, the case of MNIST image classification, where the rotation by $\pi$ will obscure the boundary between ``6'' and ``9'' \cite{augerino}.
Generally, allowing for $\beta$ to be learnable gives the model more freedom to discover various symmetries. 

The coefficient distribution $\gamma_\beta$ may be parameterized in different ways. A normal distribution centered at the origin is a natural choice, since it assigns the same probability density for a group element and its inverse, and the variance can either be fixed or learned through the reparameterization trick \cite{vae}. However, a multimodal distribution like a Gaussian mixture model may be better at capturing discrete subgroups.

However, we note that a limitation of using the Lie algebra to parameterize transformations is that it can only capture a single connected component of the Lie group. Some groups such as $\mathrm E(n)$ do not have a surjective exponential map and their group elements must be described by introducing additional discrete generators: $g=\exp[\sum_it_iL_i]\prod_jh_j$. 

\subsection{Regularization Against Trivial Solutions}
In our optimization problem \eqref{eq:gan_loss},  the generator can learn a trivial symmetry of identical transformation.
We alleviate this issue by penalizing the similarity between the input and output of the generator.  Let $R$ be similarity function on $\mathcal{X} \times \mathcal{Y}$. The regularizer is then defined
\begin{align}
    l_{\mathrm{reg}}( x, y)&=R(\Phi( x, y),( x, y)).
    \label{eq:reg}
\end{align}
We note that the similarity function need to recognize the difference in data before and after the transformation. In practice, we use cosine similarity, which only has scaling invariance in all dimensions.

Another issue arises when dealing with multi-dimensional Lie groups.  The model is encouraged to search through different directions in the manifold of the general linear group with multiple channels, i.e. $L_i$'s. In practice, however, we find that they tend to learn similar elements. To address this problem, we introduce another regularization against the channel-wise similarity, denoted as
\begin{align}
    l_{\mathrm{chreg}}(\Phi)=\sum_{1\leq i<j\leq c}R_{\mathrm{ch}}(L_i,L_j)
    \label{eq:chreg}
\end{align}
where $c$ is the number of channels in generator and $R_{\mathrm{ch}}$ is the cosine similarity for $L_i$ weights. 
We also consider setting $R_{\mathrm{ch}}$ to be the  Killing form \cite{knapp1996lie}
, a metric defined in the Lie algebra.  In this case, minimizing $l_{\mathrm{chreg}}$ corresponds to discovering an orthogonal basis for the Lie algebra. 
In practice, we find that cosine similarity works best.

Combining these regularizers with \eqref{eq:gan_loss}, we optimize the following objective:
\begin{align}
    L_\mathrm{reg}(\Phi,{D})=&\mathbb{E}_{( x, y),g}[\log D( x, y)+\log(1- D(\Phi( x, y)))\cr
    &+\lambda\cdot l_{\mathrm{reg}}( x, y)]+\eta\cdot l_{\mathrm{chreg}}(\Phi)
\end{align}

\subsection{Model Architecture}
\model~consists of two components, the generator and the discriminator. The generator simply samples a Lie group element to transform the input data and does not have any deep neural network. The discriminator can be any network architecture that fits the input. In practice, we use Multi-layer Perceptron (MLP) as a discriminator unless otherwise stated. We find that the generator loss is usually higher than the discriminator loss during training, which suggests that the generator's task of finding the correct set of symmetry transformations is harder. Therefore, a simple discriminator architecture is sufficient.

\section{Using the Discovered Symmetry}
The discovered symmetry can be used as an inductive bias to aid prediction. For instance, Augerino \cite{augerino} develops an end-to-end pipeline that simultaneously discovers invariance and trains an invariant model. For our method, there are multiple ways of utilizing the learned Lie algebra representation in downstream prediction tasks.

\subsection{Data Augmentation}
A natural idea would be augmenting the training data with the transformation generator in \model, which would lead to better generalization and robustness similar to other data augmentation approaches \cite{dao2019kernel}. To perform data augmentation in the equivariance scenario, we transform the input with group element $g$ and transform the model output with $g^{-1}$ to obtain the final prediction
\begin{align}
    \hat y = g^{-1}f_\mathrm{model}(gx).
\end{align}

\subsection{Equivariant Model}\label{sec:equiv-model}
The discovered symmetry from \model~can also be easily incorporated into existing equivariant models due to its explicit representation of the Lie algebra. This procedure is specific to different equivariant model architectures. We provide two examples of incorporating the learned symmetry into EMLP \cite{emlp} and EGNN \cite{egnn}, which are also used in experiments.

\paragraph{EMLP.} \citet{emlp} provide a simple interface for building equivariant MLPs for arbitrary matrix groups. We can directly use the discovered Lie algebra basis as input to the method and obtain an MLP equivariant to the corresponding connected Lie group, with only a minor modification to the model. The original EMLP constructs a constraint matrix according to the specified group and projects network weights to its null space. This does not work well with the symmetry discovered by \model, because it inevitably has some numerical error that results in a higher rank constraint matrix and thus lower rank weight matrix. In practice, we raise the singular value threshold to obtain an approximate equivariant subspace with more dimensions. Implementation details can be found in Appendix \ref{sec:nbody-detail}.

\paragraph{EGNN.} \citet{egnn} encode the $\mathrm E(n)$ equivariance in a graph neural network (GNN) by computing invariant edge features using the Euclidean metric. Similarly, \citet{lorentznet} develop a Lorentz invariant GNN for jet tagging by computing invariant edge features using the Minkowski metric. Both methods may be summarized as:
\begin{align}
    m_{ij}=&\phi_e(h_i,h_j,\|x_i-x_j\|_J^2,\langle x_i,x_j\rangle_J)\cr
    \text{where }& \|u\|_J =\sqrt{u^TJu},\quad \langle u,v\rangle_J=u^TJv
\label{eq:lgeb}
\end{align}
where $h_i$ and $h_j$ are scalar node features, $\|\cdot\|_J^2$ and $\langle\cdot\rangle_J^2$ are norms and inner products computed with metric tensor $J$ and $\phi_e$ is a neural network. The tensor $J$ can be varied to enforce different symmetries, such as $\mathrm{diag}(1, -1, -1, -1)$ for $\mathrm O(1,3)$ and $\mathrm{Id}_n$ for $\mathrm E(n)$. Under this formulation, the input features for $\phi_e$ are group invariant scalars, which leads to equivariance of the entire architecture.

However, the selection of metric tensor $J$ relies on knowledge of the specific symmetry group, and the application of such equivariant models is restricted if no a priori knowledge about the symmetry is readily available. We show that the discovered symmetry from \model~can replace the requirement of theoretical knowledge through a simple procedure. First, we derive an equivalent relation between an arbitrary Lie group symmetry and its invariant metric tensor (see Appendix \ref{prf:metric} for proof).
\begin{restatable}{proposition}{gim}
Given a Lie algebra basis $\{L_i\in \mathbb R^{k\times k}\}_{i=1}^c$, $\eta(u,v)=u^TJv$ ($u,v\in\mathbb R^k,J\in\mathbb R^{k\times k}$) is invariant to infinitesimal transformations in the Lie group $G$ generated by $\{L_i\}_{i=1}^c$ if and only if $L_i^TJ+JL_i=0$ for $ i=1,2,...,c$. 
\label{prop:metric}
\end{restatable}
This suggests that if we have the discovered a Lie algebra basis $\{L_i\}_{i=1}^c$, we can obtain the invariant metric tensor $J$ for the corresponding Lie group easily by solving a linear equation. However, directly solving this system gives a zero solution for $J$. Also, as the basis discovered by \model~inevitably has some numerical error, there may not be a nonzero solution. Taking these into consideration, we add a regularization term and optimize the following objective to get an approximation of ideal metric tensor
\begin{align}
    &\argmin_J \sum_{i=1}^c\|L_i^TJ+JL_i\|^2-a\cdot \|J\|^2 
\label{eq:solve-metric}
\end{align}
with $a>0$. 
The choice of regularization coefficient $a$ and the type of matrix norm can be flexible. A small push from zero is sufficient to get a reasonable metric $J$. With this approach, we can construct equivariant GNN for any discovered Lie group, which we refer to as LieGNN.

\section{Experiments}

We experiment on several tasks to demonstrate the capability of \model. Specifically, we aim to validate (1) whether \model\ can discover different types of symmetries mentioned in Table \ref{tab:comparison}; (2) whether the discovered symmetry, combined with existing models, can boost prediction performance.

\subsection{Baselines}
Direct comparison with other symmetry discovery methods is not always possible, since these works deal with different settings for discovery (see Table \ref{tab:comparison}). MSR \cite{msr} uses a largely different discovery scheme from ours and can only learn finite symmetry groups, so it is not included in the experiments. SymmetryGAN \cite{sgan} only learns an individual group element, which differs from our definition of symmetry discovery. We only include it in the first experiment to explain the difference. We mainly compare our method with Augerino \cite{augerino}, which also learns with Lie algebra representation. Augerino was originally developed for discovering a subset of a given group rather than an unknown symmetry group. We adapted Augerino from parameterizing the distribution over the given group to the distribution over the entire general linear group search space. Specifically, in the original Augerino forward function
\begin{equation}
    f_{\mathrm{aug-eq}}(x)=\mathbb{E}_{g\sim\mu}g^{-1}f(gx),
\end{equation}
we parameterize the distribution $\mu$ as in Equation \eqref{eq:lie-parameterization}. This provides ground for comparison between our method and theirs. To differentiate between this modified version with the original Augerino, we denote this approach as {Augerino+} in the following discussion.

Also, we incorporate the symmetry learned by the discovery algorithms into compatible models such as EMLP \cite{emlp} and LorentzNet \cite{lorentznet}. It should be noted that these prediction models are not directly comparable with our method since they use known symmetry whereas we focus on symmetry discovery. We combine them with \model\ to verify whether our learned symmetry representation leads to comparable prediction accuracy with the exact symmetry in theory.

\subsection{N-Body Trajectory}\label{sec:nbody}
We test our model as well as the baselines, Augerino and SymmetryGAN, on the simulated n-body trajectory dataset from Hamiltonian NN \cite{hnn}. It consists of the interdependent movements of multiple masses. We use a setting where two bodies with identical masses rotate around one another in nearly circular orbits. The task is to predict future movements based on the past series, which is rotational equivariant.
The input and output feature for each timestep has $4n$ dimensions, consisting of the positions and momentums of all bodies:
$
[q_{1x},q_{1y},p_{1x},p_{1y},...,q_{nx},q_{ny},p_{nx},p_{ny}]
$. The dataset and training details, as well as an alternative experiment setting with three bodies, are provided in Appendix \ref{sec:3body}.

We search for symmetries acting on the position and momentum of each mass separately, which induces a parameterization of $2\times2$ block diagonal matrix for the generator.

\begin{figure}[ht]
    \centering
    \begin{subfigure}{.15\textwidth}
        \includegraphics[width=\textwidth]{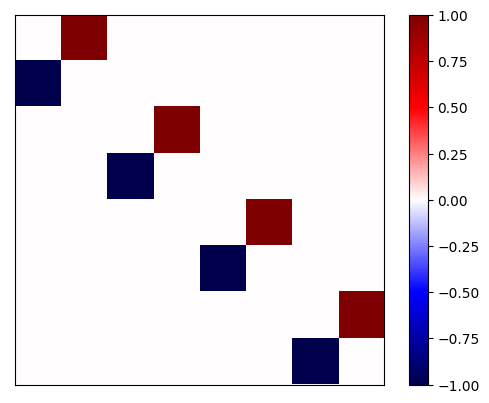}
        \caption{Ground truth}
    \end{subfigure}
    \hfill
    \begin{subfigure}{.15\textwidth}
        \includegraphics[width=\textwidth]{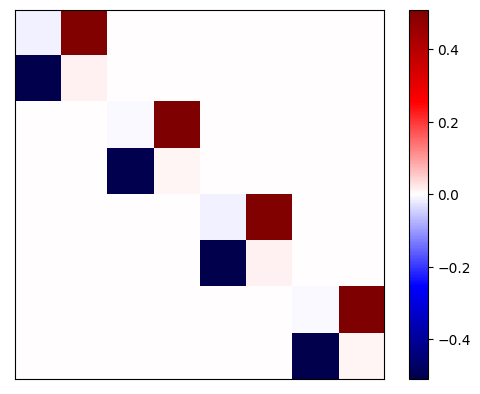}
        \caption{\model}
    \end{subfigure}
    \hfill
    \begin{subfigure}{.15\textwidth}
        \includegraphics[width=\textwidth]{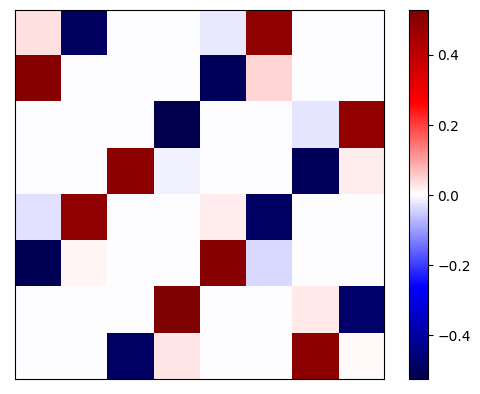}
        \caption{\model-ES}
        \label{fig:4x4}
    \end{subfigure}
    \bigskip
    \begin{subfigure}{.15\textwidth}
        \includegraphics[width=\textwidth]{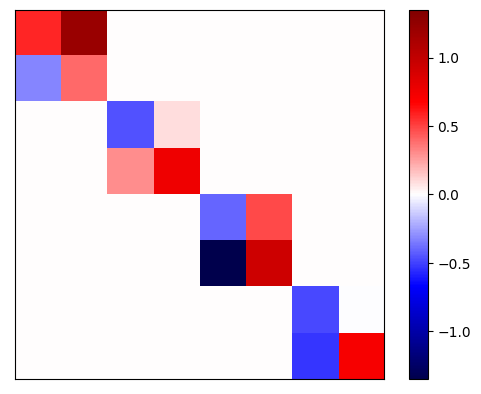}
        \caption{Augerino+}
    \end{subfigure}
    \begin{subfigure}{.15\textwidth}
        \includegraphics[width=\textwidth]{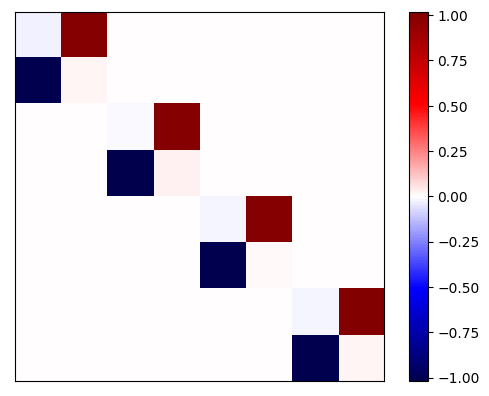}
        \caption{SymmetryGAN}
    \end{subfigure}
    \caption{Comparison between different methods on 2-body trajectory dataset. \model~discovers the correct rotation symmetry with both the original parameterization and the alternative one with expanded search space (\model-ES), whereas Augerino+ fails. SymmetryGAN only discovers one group element.}
    \label{fig:2body-vis}
\end{figure}
As is shown in Figure \ref{fig:2body-vis}, \model~can discover a symmetry that is nearly identical to ground truth, with a cosine correlation of 0.9998. We should note that the scale of the generator should not be taken into consideration when we compare different representations, because they are basis in the Lie algebra and are scale irrelevant. In contrast, Augerino+ only achieves a cosine similarity of 0.4880 with ground truth, which suggests that Augerino cannot be readily applied to discovering unknown groups.

On the other hand, SymmetryGAN \cite{sgan} produces a very similar visualization to ground truth symmetry. However, this result has a completely different interpretation. Instead of a Lie algebra generator that generates the entire group, SymmetryGAN is learning only one element of the group. In this case, it learns a rotation by $\frac{\pi}{2}$, which coincides with the Lie algebra generator.

In addition, we expand the symmetry search space of \model~to enable interactions between the position or momentum of different bodies. The result is shown in Figure \ref{fig:4x4}. Given that the origin is located at the center of mass and that the two bodies have the same mass, this can be viewed as another possible representation of the same rotation symmetry. Details of derivation for this result are included in Appendix \ref{prf:rot}.

\begin{table}[ht]
\caption{Test MSE loss of 2-body trajectory prediction. \model~and \model-ES correspond to different parameterizations of our model as is shown in Figure \ref{fig:2body-vis}. Symmetries from different discovery models and ground truth are inserted into EMLP or used to perform data augmentation. HNN is also included for camparison between equivariant models and model with other types of inductive bias.
    }
    \centering
    \begin{tabular}{c|cc}
        \toprule
        Model & EMLP & Data Aug. \\
        \midrule
        \model & 6.43e-5 & 3.79e-5 \\
        \model-ES & 2.41e-4 & 6.17e-5 \\
        Augerino+ & 9.41e-4 & 1.47e0 \\
        SymmetryGAN & - & 6.79e-4 \\
        Ground truth & \textbf{9.45e-6} & \textbf{1.39e-5} \\
        \midrule
        \midrule
        HNN & \multicolumn{2}{c}{3.63e-4} \\
        MLP & \multicolumn{2}{c}{8.49e-2} \\
        \bottomrule
    \end{tabular}
    
    \label{tab:2body-pred}
\end{table}

Besides the interpretation of the learned symmetry, we can also inject it into equivariant MLP or use it augment the training data. For prediction, The train and test datasets are constructed to have different distributions so that knowledge of symmetry would be useful for generalization. The results are shown in Table \ref{tab:2body-pred}. All experiments use the same configuration for MLP except for the introduced equivariance or data augmentation procedure. For Equivariant MLP, the two parameterizations of \model~outperform other symmetry discovery methods, approaching the performance of ground truth symmetry. MLP with no equivariance constraint can achieve lower training loss, but has trouble generalizing to a test set with the shifted distribution. For data augmentation, \model~can also achieve comparable accuracy to ground truth symmetry. SymmetryGAN only transforms the data by a fixed transformation, and its performance lies between continuous augmentation and no augmentation. 

\subsection{Synthetic Datasets}\label{sec:dr}
\begin{figure}[h!]
    \centering
    \begin{subfigure}{.18\textwidth}
        \includegraphics[width=\textwidth]{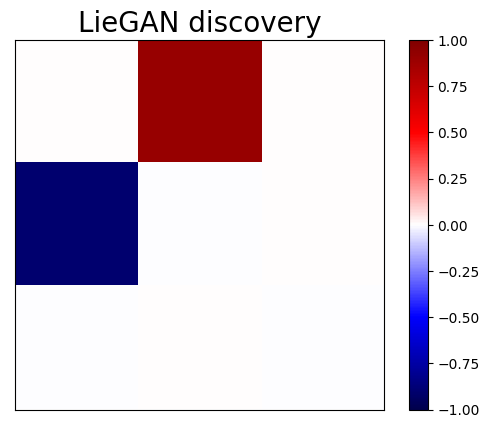}
        \caption{\model}
    \end{subfigure}
    \begin{subfigure}{.18\textwidth}
        \includegraphics[width=\textwidth]{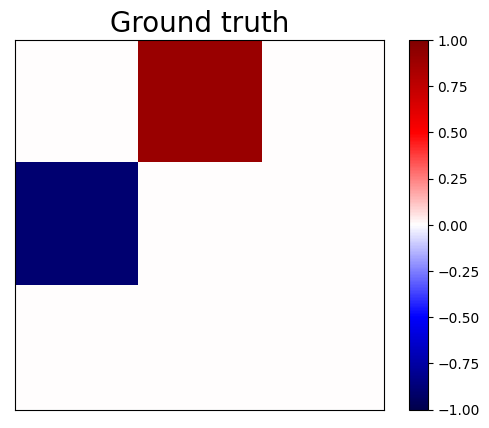}
        \caption{Ground truth}
    \end{subfigure}
    \bigskip
    \begin{subfigure}{.18\textwidth}
        \includegraphics[width=\textwidth]{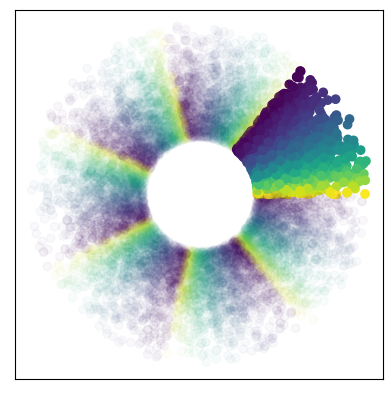}
        \caption{Original data}
    \end{subfigure}
    \begin{subfigure}{.18\textwidth}
        \includegraphics[width=\textwidth]{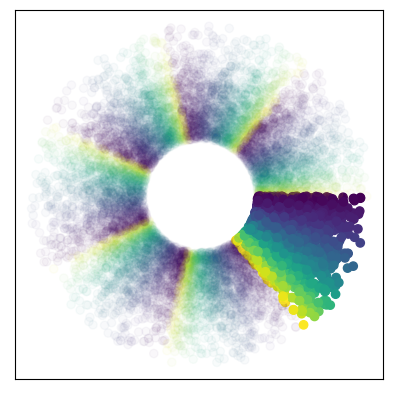}
        \caption{Transformed data}
    \end{subfigure}
    \caption{Result on the synthetic discrete rotation invariant task. (a-b): LieGAN discovers the correct rotation group and the correct scale of transformations. (c-d): Data distribution on $z=1$ plane. The color indicates the output function value. LieGAN leaves the overall data distribution unchanged while non-trivially rotating individual data points in the highlighted sector.}
    \label{fig:discrete-rotation}
\end{figure}

Next, we apply \model~to a synthetic regression problem given by $f(x,y,z)=z/(1+\arctan\frac{y}{x}\mod \frac{2\pi}{k})$. This function is invariant to rotations of a multiple of $2\pi/k$ in $xy$ plane, which form a discrete cyclic subgroup of $\mathrm{SO}(2)$ with a size of $k$. The goal is to demonstrate that our model can capture the symmetries of not only continuous Lie groups but also their discrete subgroups.

In this task, we fix the coefficient distribution to a uniform distribution on an integer grid of $[-10,10]$ to capture discrete symmetry. Figure \ref{fig:discrete-rotation} shows an example of \model~discovery in a dataset with $\mathrm{C}_7$ rotation symmetry. The discovered symmetry is almost identical to ground truth, with an MAE of 0.003. Unlike the previous case of continuous rotation, the scale of the basis matters because \model\ is modeling a set of discrete rotation symmetries with fixed angles. When acting on data, LieGAN leaves the overall data distribution unchanged while non-trivially transforms individual data points in the highlighted sector. \model\ discovers not only the rotation group but also the correct scale of transformations, which demonstrates its ability to learn a subgroup of an unknown group, which is yet another generalization from discovering the continuous symmetry of an entire Lie group.

Additional results on synthetic tasks can be found in Appendix \ref{sec:dr-more} and \ref{sec:syn-more}. For this rotation invariant task, we change the parameter $k$ to show that LieGAN can capture different discrete rotation groups. We also compare LieGAN with the baseline, SymmetryGAN, to demonstrate its advantage. Besides, other synthetic functions are designed to show that LieGAN can deal with various symmetry groups and can even work well on complex values.

\subsection{Top tagging}
We are also interested in finding symmetry groups with more complicated structures. For example, Lorentz group is an important set of transformations in many physics problems. It is a 6-dimensional Lie group with 4 connected components. While our method cannot be readily generalized to the problem of finding discrete generators, we can test whether it is capable of extracting the identity component of the Lorentz group, $\mathrm{SO}(1,3)^+$.
We use Top Quark Tagging Reference Dataset \cite{top-tagging} for discovering Lorentz symmetry, where the task is to classify between top quark jets and lighter quarks. There are 2M observations in total, each consisting the four-momentum of up to 200 particle jets. The classification task is Lorentz invariant, because a rotated or boosted input momentum should belong to the same category.

In this task, we set the generator to have up to 7 channels, which is slightly more than enough to capture the structure of 6-dimensional $\mathrm{SO}(1,3)^+$. We use cosine similarity as between-channel regularization function $l_\mathrm{chreg}$.

\begin{figure}[h]
    \centering
    \includegraphics[width=.5\textwidth]{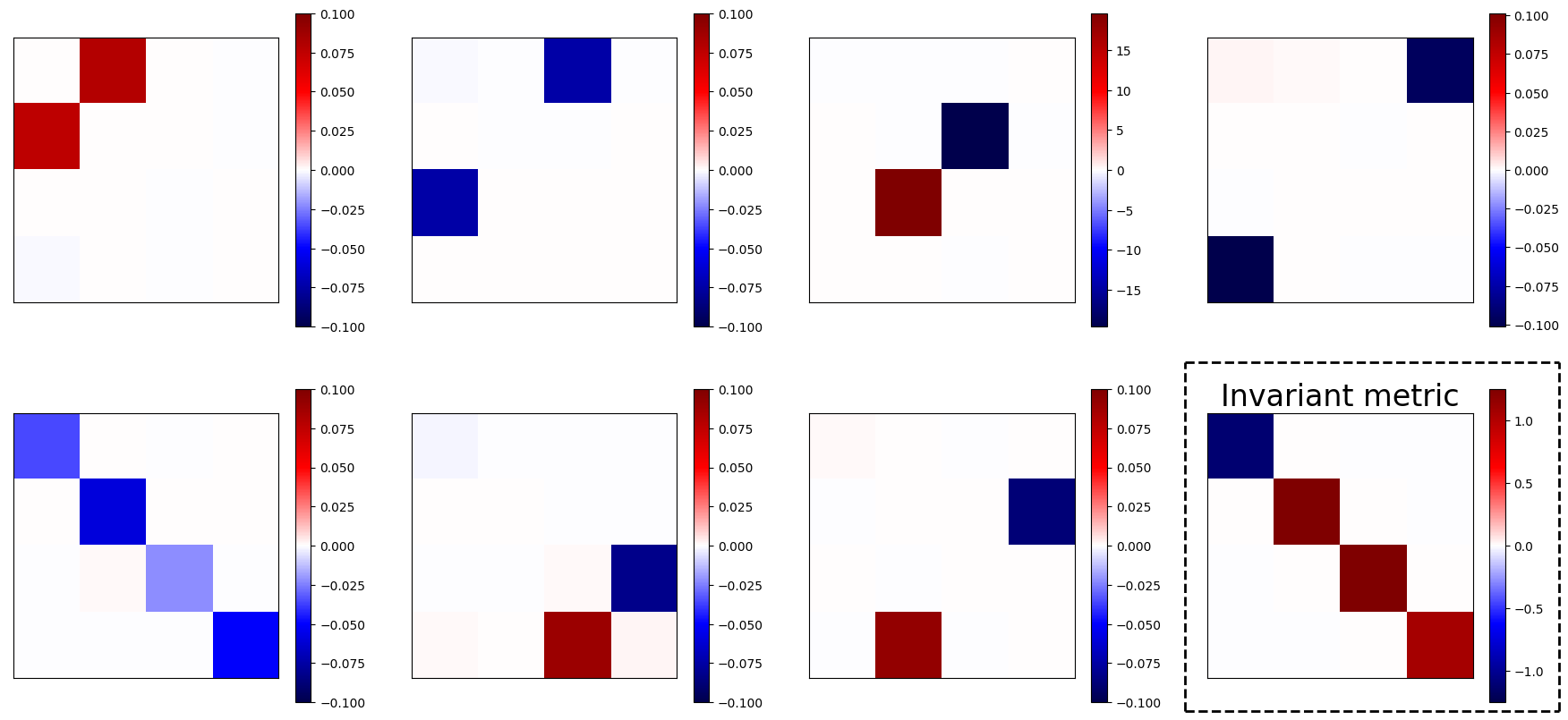}
    \caption{\model~discovers an approximate $\mathrm{SO}(1,3)^+$ symmetry in top tagging dataset, where channels 0, 1, 3 indicate boost along x-, y- and z-axis and channels 2, 5, 6 correspond to $\mathrm{SO}(3)$ rotation. Bottom-right: Computed invariant metric of the discovered symmetry by solving Equation \eqref{eq:solve-metric}.}
    \label{fig:tt-results}
\end{figure}

The discovery results are shown in Figure \ref{fig:tt-results}. The four dimensions in the matrix correspond to the 4-momentum $(E/c,p_x,p_y,p_z)$. \model\ is successful in recovering the $\mathrm{SO}(1,3)^+$ group. Its channels 2, 5, 6 correspond to $\mathrm{SO}(3)$ rotation, and channels 0, 1, 3 indicate boost along x-, y- and z-axis. In addition, the generator learns an additional Lie algebra element that scales different input dimensions with approximately the same amounts.

\begin{figure}[h]
    \centering
    \begin{subfigure}{.23\textwidth}
        \includegraphics[width=\textwidth]{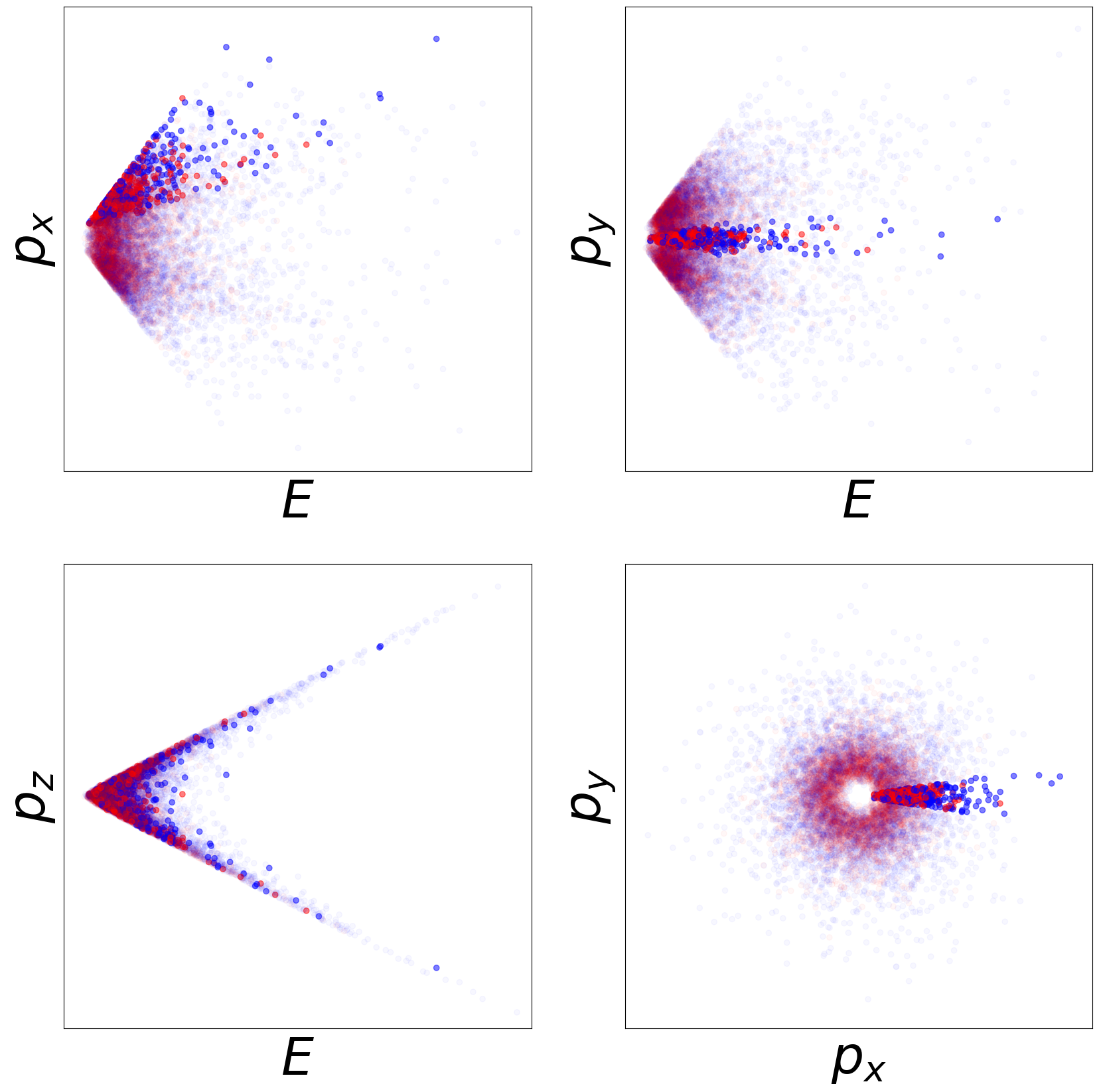}
        \caption{Original}
    \end{subfigure}
    \begin{subfigure}{.23\textwidth}
        \includegraphics[width=\textwidth]{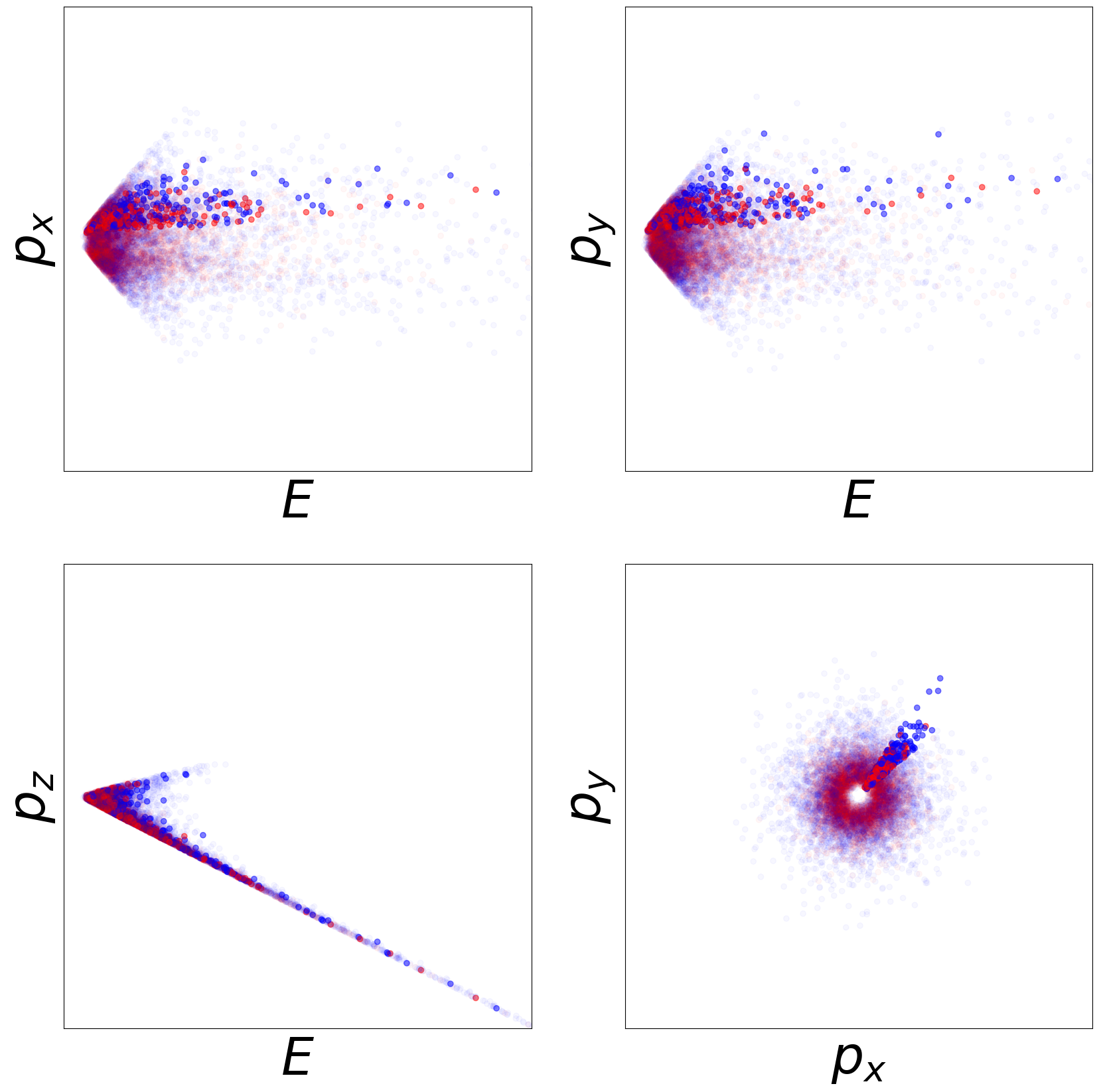}
        \caption{Transformed}
    \end{subfigure}
    \caption{The data distribution before and after the LieGAN transformations. The overall distribution remains unchanged, while the highlighted data points are non-trivially transformed.}
    \label{fig:tt-dist}
\end{figure}

Besides, figure \ref{fig:tt-dist} visualizes the distribution of the leading jet component in each event before and after LieGAN transformations. For better demonstration, four 2D marginal distributions of $(E,p_x),(E,p_y),(E,p_z),(p_x,p_y)$ are plotted. The overall distribution remains unchanged, while the data points in the highlighted portions are rotated and boosted to new locations. These results suggest that \model\ is capable of discovering high-dimensional Lie groups and also decoupling the group structure to a simple and interpretable representation of Lie algebra basis. 

\begin{table}[]
\caption{Test accuracy and AUROC on top tagging. Our proposed model, LieGNN, reaches the performance with LorentzNet which explicitly encodes Lorentz symmetry. The result of non-equivariant GNN (LorentzNet (w/o)) and EGNN is from \citet{lorentznet}.
    }
    \centering
    \begin{tabular}{c|c|c}
        \toprule
        Model & Accuracy & AUROC \\
        \midrule
        LorentzNet & \textbf{0.940} & \textbf{0.9857} \\
        LieGNN & 0.938 & 0.9848 \\
        LorentzNet (w/o) & 0.934 & 0.9832 \\
        EGNN & 0.922 & 0.9760 \\
        \bottomrule
    \end{tabular} 
    \label{tab:tt-pred}
\end{table}
It is also possible to inject this knowledge of Lie group symmetry into existing prediction models. Following the guideline in Section \ref{sec:equiv-model}, we compute the invariant metric of the discovered symmetry (Figure \ref{fig:tt-results} bottom-right), which is almost identical to the true Minkowski metric, with a cosine correlation of $-0.9975$. The computed metric is used to construct the LieGNN equivariant to the discovered group. Table \ref{tab:tt-pred} shows the prediction results. Without requiring any prior knowledge, LieGNN with the metric derived from \model\ discovery reaches the performance of LorentzNet \cite{lorentznet} with the true Minkowski metric.

\section{Conclusion}
In this paper, we present a method of discovering symmetry from training dataset alone with generative adversarial network. Our proposed framework addresses the discovery of various symmetries, including continuous Lie group symmetries and discrete subgroup symmetries, which is a significant step forward compared to existing symmetry discovery methods with relatively narrow search space for symmetry. We also develop pipelines for utilizing the learned symmetry in downstream prediction tasks through equivariant model and data augmentation, which proves to improve prediction performance on a variety of datasets.

This work currently deals with global symmetry of subgroups of general linear groups. However, it is also possible to apply this framework to more general scenario of symmetry discovery, such as non-connected Lie group symmetry, nonlinear symmetry and gauge symmetry, by replacing the simple linear transformation generator in \model~with more sophisticated structure. For instance, nonlinear symmetry could possibly be found by adding layers in generator to project the input to a space with linear symmetry.

Moreover, \model~shows tremendous potential in its application to supervised prediction tasks, which suggests that automatic symmetry discovery methods may eventually replace the need of human prior knowledge about symmetry. However, this ultimate vision can be fully realized only if equivariant neural network models can be implemented for more general choices of symmetry groups rather than a few specific symmetries. We have demonstrated in this work how to incorporate the discovered symmetry into some equivariant models including EMLP and EGNN, which we hope could inspire further exploration in this topic.

\section*{Acknowledgement}
This work was supported in part by the U.S. Department Of Energy, Office of Science, U. S. Army Research Office
under Grant W911NF-20-1-0334, Google Faculty Award, Amazon Research Award, and NSF Grants \#2134274, \#2107256 and \#2134178.

\bibliography{ref}
\bibliographystyle{icml2022}

\newpage
\appendix
\onecolumn


\section{Proofs}
\subsection{Optimizng the GAN loss function}\label{prf:gan}
We show in this section how optimizing the GAN loss function can lead to proper symmetry discovery. We assume in the first place the existence of a symmetry group and derive the properties of loss when the generator learns this group or its subgroup. We use the definition of perfect symmetry, that is, $gf(x)=f(gx)$ or $P_d(gf(x)|gx)=1$ for symmetry group elements $g$. We use $p_d$ and $p_{gen}$ to denote the original distribution of data and the generated distribution.

\begin{assumption}\label{assum:existence}
There exists a maximal subgroup of $GL(n)$, denoted as $G^*$, which $y=f(x)$ is equivariant to. That is, $\forall g\in G^*,gy=f(gx)$; $\forall g\in GL(n)\backslash G^*,p_d(gy\not=f(gx))>0.$
\end{assumption}

\strong*

\begin{proof}
Revisiting Eq \eqref{eq:p_gen}, the generated distribution is given by
\begin{align}
    p_{gen}(x,y)=&\int_{G^*}\mu(g)p_d(g^{-1}x)p_d(g^{-1}y|g^{-1}x)/|\rho_\mathcal X (g)||\rho_\mathcal{Y}(g)|dg
\end{align}
If $p_d(x)$ is proportionally distributed along each orbit of $G^*$-action on $\mathcal X$, then
\begin{align}
    p_{gen}(x,y)=&\int_{G^*}\mu(g)p_d(x)p_d(g^{-1}y|g^{-1}x)dg
\end{align}

For any group element $g\in G^*$, $p_d(y|x)=p_d(g^{-1}y|g^{-1}x)$. Therefore,
\begin{align}
    p_{gen}(x,y)=&(\int_{G^*}\mu(g)dg)p_d(x)p_d(y|x)\\
    =&p_d(x,y)
\end{align}

As this equality holds for all $(x,y)$, we have zero divergence between these two distributions, $p_d$ and $p_{gen}$.
\end{proof}
While this distribution condition is often not satisfied in practice, we further show that under certain assumptions on data and an ideal discriminator, a nontrivial Lie subgroup of the true symmetry group corresponds to a local minimum of generator loss function.

\begin{assumption}\label{assum:neg}
For each datapoint from the original distribution, transformations outside that maximal subgroup $G_T$ on it would not produce a valid datapoint. Formally, denoting $\bar{G^*}=GL(n)\backslash G^*$, $\int_{\bar{G^*}} \mu(g)p_d(gy|gx)dg=0$. (While there might be slim chances that $gf(x)=f(gx)$ for some $g\in\bar{G^*}$, the integration can still yield zero as long as we parameterize $\mu(g)$ with good properties.) 
\end{assumption}

\begin{assumption}\label{assum:uniformity}
For each orbit $[x]$ of $G^*$ with $P_d([x])>0$, $\exists x_0\in[x],c>0,m>0$ s.t. $\forall {g\in\delta_0(m)}$, $ p_d(gx_0)\geq c$, $ p_d(g^2x_0)\geq c$, and $V(g)=|\rho_\mathcal X(g)||\rho_\mathcal Y(g)|\in (v_m, V_m)$, where $\delta_0(m)$ is a neighborhood of $\mathrm{id}$ with $P_\mu(\delta_0(m))=m$ and $V_m>v_m>0$ are constants depending on $m$.
\end{assumption}

This is actually a much more relaxed version of distribution constraint along the group action orbits in Theorem \ref{thm:strong}, which may be unrealistic. We assume instead that there exists a continuous neighborhood in each orbit where the density of $x$ is above some threshold.

\weak*

\begin{proof} As an established result in GAN, the optimal discriminator for the loss function \eqref{eq:gan_loss} is
\begin{equation}
    D^*(x,y)=\frac{p_d(x,y)}{p_d(x,y)+p_{gen}(x,y)} \label{eq:optimal_D}
\end{equation}
Substituting \eqref{eq:optimal_D} into \eqref{eq:gan_loss}, we get
\begin{align}
    L(\Phi,D^*)=&\int p_d(x,y)\log\frac{p_d(x,y)}{p_d(x,y)+p_{gen}(x,y)}+p_{gen}(x,y)\log\frac{p_{gen}(x,y)}{p_d(x,y)+p_{gen}(x,y)}dxdy\\
    =&\int_{p_d(x,y)\not=0}p_d(x,y)\log\frac{p_d(x,y)}{p_d(x,y)+p_{gen}(x,y)}+p_{gen}(x,y)\log\frac{p_{gen}(x,y)}{p_d(x,y)+p_{gen}(x,y)}dxdy\label{eq:loss_under_D*}
\end{align}
where, denoting $\tilde\mu(g)=\mu(g)/|\rho_\mathcal X (g)||\rho_\mathcal{Y}(g)|$,
\begin{equation}
    p_{gen}(x,y)=\int_g\tilde\mu(g)p_d(g^{-1}x)p_d(g^{-1}y|g^{-1}x)dg
\end{equation}
Because the Haar measure $dg$ is invariant to inversion, we have
\begin{equation}
    p_{gen}(x,y)=\int_g\tilde\mu(g^{-1})p_d(gx)p_d(gy|gx)dg
\end{equation}
In practice, we use Gaussian distribution for $\mu(g)$, which assigns the same probability for a group element and its inverse. (This is also true for many other common choices of distribution, such as uniform distribution centered at origin.) Therefore, denoting $V(g)=|\rho_\mathcal X(g)||\rho_\mathcal Y(g)|$,
\begin{align}
    p_{gen}(x,y)=&\int_g\mu(g)p_d(gx)p_d(gy|gx)|\rho_\mathcal X(g)||\rho_\mathcal Y(g)|dg\\
    =&\int_g\mu(g)p_d(gx)p_d(gy|gx)V(g)dg
\end{align}
It is easy to show that the Lie group generated by the intersection of two Lie algebras coincides with the intersection of Lie groups generated by these two Lie algebras, respectively. Therefore, as $\mathfrak{g}_2\cap\mathfrak{g}_T=\{\mathrm{\mathbf0}\}$, $G_2\cap G_T=\{\mathrm{id}\}$.
According to Assumption 2,
\begin{align}
    p_{gen}(x,y;G_2)=&\int_g\mu(g)p_d(gx)p_d(gy|gx)V(g)dg\\
    =&\int_{g\in\delta_0(1-\eta)}\mu(g)p_d(gx)p_d(gy|gx)V(g)dg+\cr
    &\int_{g\notin\delta_0(1-\eta)}\mu(g)p_d(gx)p_d(gy|gx)V(g)dg\\
    \leq&\int_{g\in\delta_0(1-\eta)}M\mu(g)p_d(gy|gx)dg+\cr
    &\int_{g\notin\delta_0(1-\eta)}\mu(g)p_d(gx)V(g)dg\\
    =&0+\int_{g\notin\delta_0(1-\eta)}\mu(g)p_d(gx)V(g)dg
\end{align}
where, following the notations in Assumption 3, $\delta_0(1-\eta)$ is the neighborhood of $\mathrm{id}$, $P_\mu(\delta_0(1-\eta))=1-\eta$, and $V(g)\leq V_{1-\eta}$. Therefore, there exists an upper bound $M=\max_{g\in\delta_0(1-\eta)}p_d(gx)V(g)$.

For the integral on $g\notin\delta_0(1-\eta)$, as the Gaussian density $\mu(g)$ decays exponentially with $V(g)$ and $p_d(gx)$ has an upper bound, $\forall\epsilon>0,\exists\eta$ s.t. $\int_{g\notin\delta_0(1-\eta)}\mu(g)p_d(gx)V(g)dg<\epsilon$.

Therefore, $p_{gen}(x,y;G_2)=0$ and $L(\mathfrak{g}_2,D^*)=0$.

On the other hand, $\mathfrak{g}_1\cap\mathfrak{g}^*\not=\{\mathrm{\mathbf0}\}\Rightarrow G_1\cap G^*\not=\{\mathrm{id}\}$. We consider the integral \eqref{eq:loss_under_D*} along each possible orbit of $G^*$. According to Assumption 3, there exists an x-neighborhood $X=\delta_0(m)x_0$ s.t. $\forall x\in X,p_d(x,f(x))>c$. For the generated distribution on this neighborhood, we have
\begin{align}
    p_{gen}(x,f(x))=&\int_g\mu(g)p_d(gx)V(g)dg\\
    \geq& \int_{g\in\delta_0(m)}\mu(g)p_d(gx)V(g)dg\\
    \geq& \int_{g\in\delta_0(m)}\mu(g)cv_mdg\\
    =&mcv_m>0
\end{align}

As the supports of $p_d$ and $p_{gen}$ overlap on this neighborhood, we have $L(\mathfrak{g}_1,D^*)<0=L(\mathfrak{g}_2,D^*)$.
\end{proof}

\subsection{Experiment Result on 2-Body Trajectory Dataset}\label{prf:rot}
In Figure \ref{fig:4x4}, we observe an unfamiliar symmetry representation. In fact, this is another possible representation for rotation symmetry. The learned Lie algebra basis $L$ can be expressed in the following form after discarding the noise:
\begin{align*}
    R&=\begin{bmatrix}
0 & -1\\
1 & 0
\end{bmatrix}\\
L&=\begin{bmatrix}
R & & -R & \\
 & R & & -R\\
-R & & R & \\
 & -R & & R
\end{bmatrix}
\end{align*}
Computing the matrix exponential gives
\begin{align*}
    \exp(\theta L)&=\begin{bmatrix}
L(\theta) & & -L(\theta) & \\
& L(\theta) & & -L(\theta) \\
-L(\theta) & & L(\theta) & \\
& -L(\theta) & & L(\theta) \\
\end{bmatrix}+I\\
L(\theta)&=\sum_{k=0}^{+\infty}\frac{2^{2k}(-1)^k\theta^{2k+1}}{(2k+1)!}R+\sum_{k=1}^{+\infty}\frac{2^{2k-1}(-1)^{k}\theta^{2k}}{(2k)!}I
\end{align*}
As the origin for this dataset is at the center of mass and $m_1=m_2$, we have $\bold q_1=-\bold q_2$ and $\bold p_1=-\bold p_2$. Therefore,
\begin{align*}
    \exp(\theta L)\begin{bmatrix}
\bold q_1 \\ \bold p_1 \\ \bold q_2 \\ \bold p_2
\end{bmatrix}&=\mathrm{diag}(I+2L(\theta))\begin{bmatrix}
\bold q_1 \\ \bold p_1 \\ \bold q_2 \\ \bold p_2
\end{bmatrix}\\
I+2L(\theta)&=\sum_{k=0}^{+\infty}\frac{2^{2k+1}(-1)^k\theta^{2k+1}}{(2k+1)!}R+\sum_{k=0}^{+\infty}\frac{2^{2k}(-1)^{k}\theta^{2k}}{(2k)!}I\\
&=\begin{bmatrix}
\cos2\theta & -\sin2\theta\\
\sin2\theta & \cos2\theta
\end{bmatrix}
\end{align*}
which indicates that this is another representation for rotation specific to this dataset.

\subsection{Computing Group Invariant Metric Tensor}\label{prf:metric}
\gim*
\begin{proof}
An infinitesimal transformation in group $G$ generated by the given Lie algebra basis can be written as the matrix representation
$g=I+\sum_i\epsilon_iL_i$.
\begin{align}
&\eta(u,v)=\eta(gu,gv)\\
\Longleftrightarrow & u^TJv=u^Tg^TJgv\\
\Longleftrightarrow & u^T(I+\sum_i\epsilon_iL_i^T)J(I+\sum_i\epsilon_iL_i)v=u^TJv\\
\Longleftrightarrow & u^TJv+u^T(\sum_i\epsilon_i(L_i^TJ+JL_i))v+O(\epsilon^2)=u^TJv\\
\Longleftrightarrow & u^T(\sum_i\epsilon_i(L_i^TJ+JL_i))v=0, \forall u,v\in\mathbb R^n
\end{align}
As this holds for any infinitesimal transformation $g$, we can set $\epsilon_{-i}=0$ to get
\begin{equation}
    \epsilon_i(L_i^TJ+JL_i)=0,i=1,2,...,c
\end{equation}
Therefore, $L_i^TJ+JL_i=0,i=1,2,...,c$.

On the other hand, if $L_i^TJ+JL_i=0,i=1,2,...,c$, then
\begin{equation}
    \sum_i\epsilon_i(L_i^TJ+JL_i)=0,\forall \epsilon\in\mathbb R^c
\end{equation}
\end{proof}

\newpage
\section{Experiment Details}
This section provides detailed explanation on the experiment settings, including the dataset generating procedure, the hyperparameters used in training, etc.
\subsection{N-Body Trajectory}\label{sec:nbody-detail}
We use the code from Hamiltonian Neural Networks \footnote{\href{https://github.com/greydanus/hamiltonian-nn}{https://github.com/greydanus/hamiltonian-nn}} \cite{hnn} to generate the dataset for this task. We construct the train and test sets with different distributions to test the generalization ability of the models. Specifically, we sort the samples in terms of the polar angle of the position of the first particle at the starting timestep of trajectory, and divide the sorted dataset into train and test sets.

The task for this dataset is to predict $K$ future timesteps of 2-body movement based on $P$ past timesteps of observation, where the feature for each timestep has 8 dimensions, consisting of the positions and momentums of two bodies: $[q_{1x},q_{1y},p_{1x},p_{1y},q_{2x},q_{2y},p_{2x},p_{2y}]$. In our experiment, we set $P=K=5$. When discovering symmetry, the \model~generator takes both the past observations and future predictions as input, yielding an input dimension of 80. The generator transforms each timestep at the same time, which means that it is learning a group representation of $\mathbb R^{8\times 8}$ that acts simultaneously on each of the past input and future output timesteps. On the other hand, we use a 3-layer MLP as with discriminator, with input dimension 80, hidden dimension 512, and leaky ReLU activation with negative slope 0.2. We use only the regularization against identical transformations, i.e. $l_\mathrm{reg}(x,y)$ in \eqref{eq:reg},  with the regularization coefficient $\lambda=1$. but not the between-channel regularization in \eqref{eq:chreg}, because the generator only has a single channel and there is no need for it. The learning rates for the discriminator and the generator are set to 0.0002 and 0.001, respectively. \model~is trained adversarially for 100 epochs.

In the prediction task with equivariant model, we use EMLP with 3 hidden layers and a hidden representation of $5V$, where $V$ stands for an 8-dimensional vector just as the feature for each timestep. We train all EMLPs constructed with different equivariances with lr=0.0001 for 5000 epochs.

\begin{figure}[h]
    \centering
    \includegraphics[width=.6\textwidth]{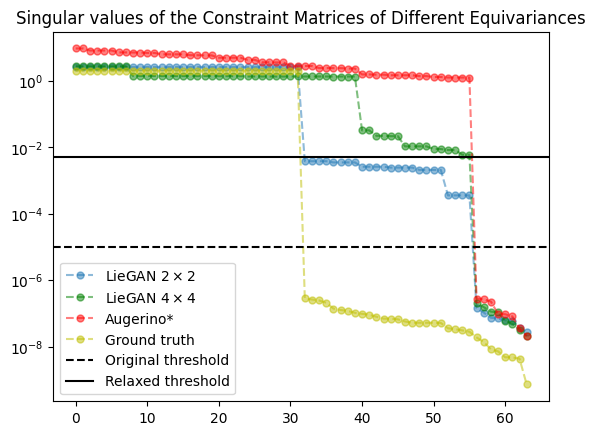}
    \caption{Singular values of the EMLP constraint matrices derived from different equivariances under the representation of group actions on weight matrices mapping from $V_1\rightarrow V_2$, where $V_1$ and $V_2$ are both 8-dimensional vector spaces. The $y$ axis is log-scaled for better visualization. It can be observed that the singular values of the constraint matrix corresponding to the symmetry discovered by \model~have a sharp decrease at the same position as the matrix for ground truth symmetry. This suggests that we can slightly relax the singular value threshold to obtain a higher dimensional equivariant subspace.}
    \label{fig:svd}
\end{figure}
As is mentioned in Section \ref{sec:equiv-model}, we slightly modified the EMLP implementation to adapt to the noised discovery result from \model. EMLP projects the network weight to an equivariant subspace, which is the null space of the constraint matrix derived from the provided equivariance and input and output representations. The null space is computed with SVD. This method usually works for common handpicked symmetries, such as Euclidean group and Lorentz group, which typically have sparse and clean matrix representations. However, the symmetry discovered by \model\ inevitably has some numerical error. While such error could be largely negligible when we visualize the discovered symmetry or use it for data augmentation, it will cause problem in the SVD procedure in the EMLP implementation. Even a small noise that changes a matrix representation entry from zero to small nonzero values could result in a constraint matrix with higher rank, which then leads to a lower dimensional equivariant subspace and a lower rank weight matrix. However, we can raise the singular value threshold to larger values to calculate an approximate null space, which has higher dimensions. Figure \ref{fig:svd} shows how we modify the singular value threshold. The original EMLP implementation sets a threshold of 1e-5. With this threshold, the symmetries discovered by \model~lead to a weight matrix that maps each input vector to each hidden vector with a rank of 8, significantly lower than 32, which is the case for ground truth rotation symmetry. However, it can also be observed that the singular values of the constraint matrix corresponding to \model~symmetry have a sudden fall at the same position as the matrix for ground truth symmetry. Therefore, we can raise the singular value threshold to 5e-3, which is still reasonably small, and obtain a 32-dimensional approximately equivariant subspace for the discovered symmetry. This procedure proves to significantly improve prediction performance for EMLP constructed with the discovered symmetry.

\subsection{Synthetic Regression}
This is a regression problem given by $f(x,y,z)=z/(1+\arctan\frac{y}{x}\mod \frac{2\pi}{k})$. This function is invariant to rotations of a multiple of $2\pi/k$ in $xy$ plane, which form a discrete cyclic subgroup of $\mathrm{SO}(2)$ with a size of $k$. In our experiment, we construct the dataset with $k=7$. We randomly sample 20000 inputs $(x,y,z)$ from a standard multivariate Gaussian distribution and calculates the output analytically. For symmetry discovery, we use a generator with a single channel of $\mathbb R^{3\times 3}$ matrix representation and a 3-layer MLP discriminator with input dimension 4 (which is $(x,y,z,f)$), hidden dimension 512, and leaky ReLU activation with negative slope 0.2. The coefficient distribution in the generator is set to a uniform distribution on integer grid between $[-10,10]$. We use regularization term $l_\mathrm{reg}$ with coefficient $\lambda=0.01$. The learning rates for the discriminator and the generator are set to 0.0002 and 0.001, respectively. \model~is trained for 100 epochs.

\subsection{Top Quark Tagging}
For symmetry discovery, we use a generator with 7 channels of $\mathbb R^{4\times4}$ matrix representations acting on the input 4-momenta $(E/c, p_x, p_y, p_z)$. The input consists of the momenta of up to 200 constituents for each sample, sorted by the transverse momentum of each constituents. We truncate the input to the momenta of the two leading constituents, which gives an input dimension of 8. As this classification task is invariant, the generator does not change the category label associated with each sample. The discriminator takes both the transformed input momenta $gx$ and the output label $gy=y$ as its input. It first transforms $y$ to a real-valued vector with an embedding layer, and then concatenates the embedding with $gx$, and passes them through a 3-layer MLP with hidden dimension 512 and leaky ReLU activation with negative slope 0.2. We use regularizations $l_\mathrm{reg}$ with coefficient $\lambda=1$ and $l_\mathrm{chreg}$ with coefficient $\eta=0.1$. The learning rates for the discriminator and the generator are set to 0.0002 and 0.001, respectively. \model~is trained for 100 epochs.

For prediction with LieGNN, we first calculate the invariant metric tensor based on the discovered symmetry according to Equation \eqref{eq:solve-metric}. We optimize the objective with $a=0.0005$ and matrix max norm using a gradient descent optimizer with step size of $1\times10^{-5}$. Then, we build the LieGNN prediction model based on LorentzNet implementation \footnote{\href{https://github.com/sdogsq/LorentzNet-release}{https://github.com/sdogsq/LorentzNet-release}}. The model has 6 group equivariant blocks with 72 hidden dimensions. We use a dropout rate of 0.2 and weight decay rate of 0.01. The model is trained for 35 epochs with a learning rate of 0.0003. These settings are the same for LorentzNet and LieGNN.

\newpage
\section{Additional Experiments}
\subsection{N-Body Trajectory}\label{sec:3body}
We extend the 2-body setting in Section \ref{sec:nbody} to 3-body movements. Despite the increased complexity, \model\ is still able to discover the rotation symmetry in this case, as is shown in Figure \ref{fig:3body-discovery}.
\begin{figure}[h]
    \centering
    \includegraphics[width=.35\textwidth]{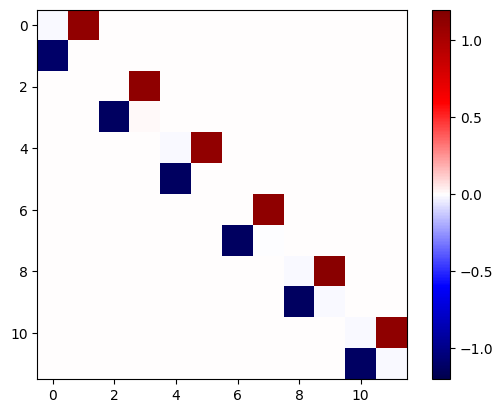}
    \caption{Symmetry discovery result on 3-body trajectory prediction dataset. \model\ can also learn an accurate representation of rotation symmetry as in the case of 2-body dataset.}
    \label{fig:3body-discovery}
\end{figure}

\subsection{Synthetic Regression}\label{sec:dr-more}
Consider the function $f(x,y,z)=z/(1+\arctan\frac{y}{x}\mod \frac{2\pi}{k})$ introduced in Section \ref{sec:dr}. We use different values of $k$ to construct functions that are invariant to different groups. Table \ref{tab:discrete-rotation-more-k} shows the results for $k=6,7,8$, corresponding to the cyclic groups $\mathrm{C}_6,\mathrm{C}_7,\mathrm{C}_8$. LieGAN successfully captures these discrete rotation groups. SymmetryGAN works fine in some cases, but its convergence heavily depends on random initialization. For example, it does not converge for $k=8$ and converges to a non-generator element $R(4\pi/3)$ for $k=6$.
\begin{table}[ht]
    \centering
    \begin{tabular}{c|ccc}
    \hline
        $k$ & 6 & 7 & 8 \\
    \hline
        LieGAN & 0.012 & 0.003 & 0.011 \\
        SymmetryGAN & 0.024* & 0.034 & N/A\\
    \hline
    \end{tabular}
    \caption{Mean absolute error between the discovered symmetry representations and ground truths.}
    \label{tab:discrete-rotation-more-k}
\end{table}

\subsection{More Synthetic Tasks}\label{sec:syn-more}
\paragraph{Partial permutation symmetry.} Consider the function $f(x)=x_1+x_2+x_3+x_4^2-x_5^2,\ x\in\mathbb R^5$. It has partial permutation symmetry, i.e. the output stays the same if we permute the first 3 dimensions of $x$, but it will change if we also permute the last 2 dimensions. As permutation is a discrete symmetry, we set the coefficient distribution to a uniform distribution on an integer grid, similar to the setting in Section \ref{sec:dr}.
\begin{figure}[h]
    \centering
    \includegraphics[width=.75\textwidth]{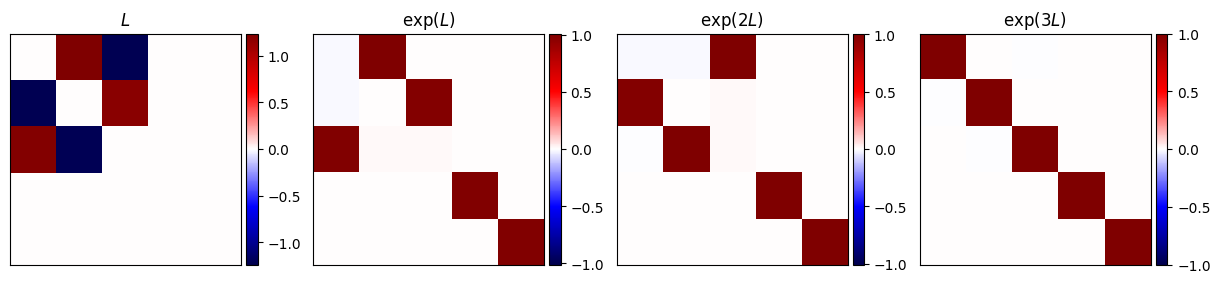}
    \caption{Discovered partial permutation symmetry.}
    \label{fig:perm}
\end{figure}

Figure \ref{fig:perm} shows the discovery result. The LieGAN generator exactly matches the ground truth,
$$
L_\text{truth}=\log(\begin{bmatrix}
    0 & 1 & 0 & 0 & 0 \\
    0 & 0 & 1 & 0 & 0 \\
    1 & 0 & 0 & 0 & 0 \\
    0 & 0 & 0 & 1 & 0 \\
    0 & 0 & 0 & 0 & 1
\end{bmatrix})\approx1.21\times\begin{bmatrix}
    0 & 1 & -1 & 0 & 0 \\
    -1 & 0 & 1 & 0 & 0 \\
    1 & -1 & 0 & 0 & 0 \\
    0 & 0 & 0 & 0 & 0 \\
    0 & 0 & 0 & 0 & 0
\end{bmatrix},
$$
with $\mathrm{MAE}=0.003$. The figure also shows that when we compute the exponential of $L$, $2L$ and $3L$, we get the permutations $(123),\ (132)$ and $\mathrm{id}$.

\paragraph{$\mathrm{SU}(2)$ symmetry.} In this example, we show that LieGAN can also work on complex-valued tasks. Consider the function $f(x,y)=\frac{1}{2}(x_1y_2-x_2y_1)^2+(x_1y_2-x_2y_1),\ x,y\in\mathbb C^2$. Such holomorphic functions are referred to as superpotentials which are relevant to supersymmetric field theories \cite{krippendorf2020detecting}. We want to find a complex Lie algebra representation, i.e. $\{L_i\in\mathbb C^{2\times2}\}_{i=1}^c$, that acts on the inputs, $x$ and $y$, simultaneously. The true underlying invariance here is the special unitary group, $\mathrm{SU}(2)$.
\begin{figure}[h]
    \centering
    \includegraphics[width=.65\textwidth]{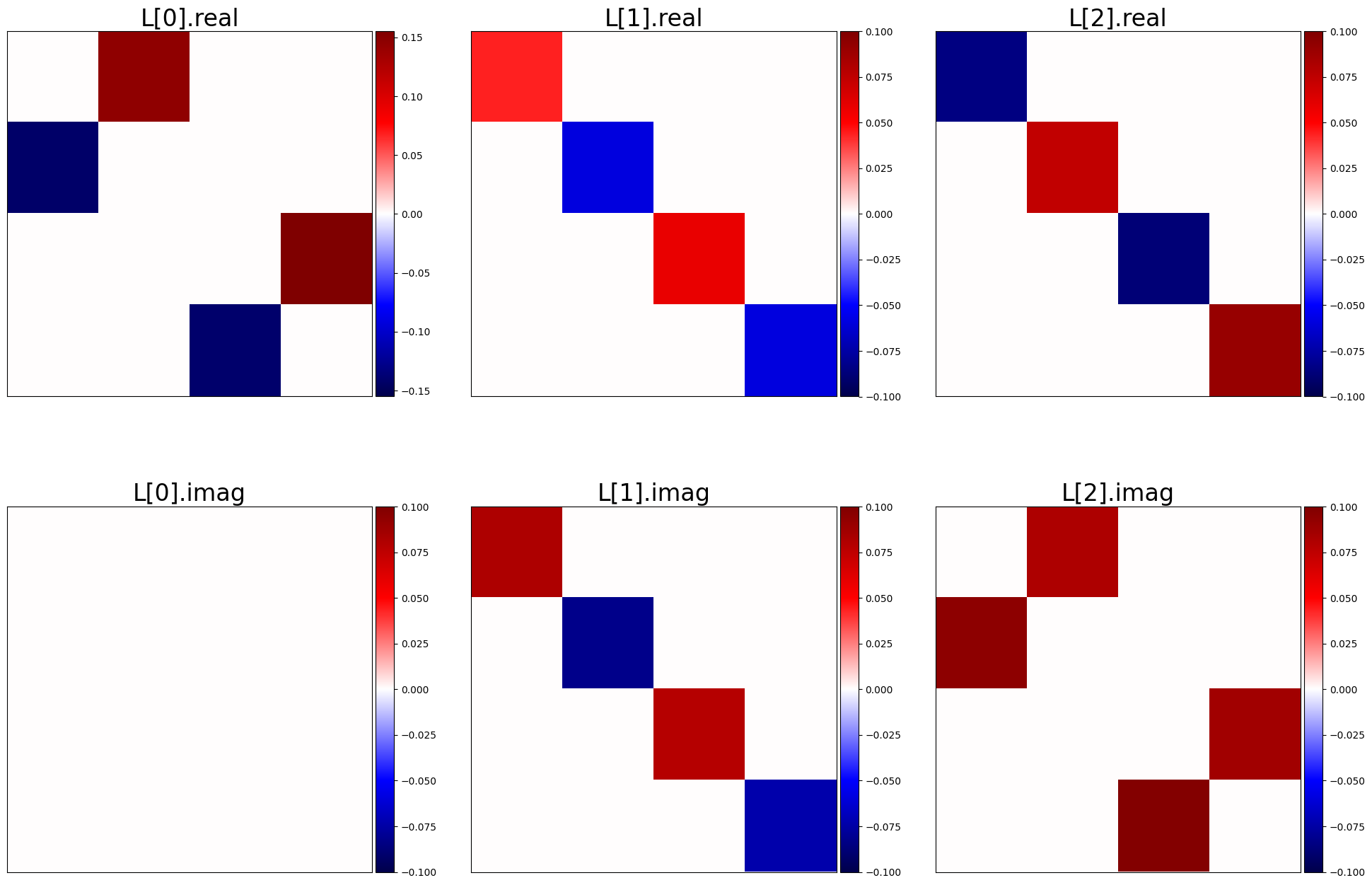}
    \caption{The complex Lie algebra discovered by LieGAN.}
    \label{fig:su2}
\end{figure}

We set the number of generator channels to $c=3$. Figure \ref{fig:su2} shows the discovered complex Lie algebra. It can be approximately written in the following numerical form:
$$
L_1=\begin{bmatrix} 0 & 1 \\ -1 & 0\end{bmatrix},\ L_2=\begin{bmatrix} 1+c_1i & 0 \\ 0 & -1-c_1i\end{bmatrix},\ L_3=\begin{bmatrix} -1 & c_2i \\ c_2i & 1\end{bmatrix}
$$
A more familiar form of $\mathfrak{su}(2)$ representation is given by
$$
u_1=\begin{bmatrix} 0 & -1 \\ 1 & 0\end{bmatrix},\ u_2=\begin{bmatrix} i & 0 \\ 0 & -i\end{bmatrix},\ u_3=\begin{bmatrix} 0 & i \\ i & 0\end{bmatrix}
$$
It can be easily checked that our discovery result is equivalent to this representation upon a change of basis:
$$
\begin{bmatrix}
    u_1 \\ u_2 \\ u_3
\end{bmatrix}=
\begin{bmatrix}
    -1 & & \\
     & \frac{i}{1+c_1i} & \\
     & \frac{1}{c_2(1+c_1i)} & \frac{1}{c_2}
\end{bmatrix}\begin{bmatrix}
    L_1 \\ L_2 \\ L_3
\end{bmatrix}
$$
Thus, we may conclude that LieGAN can identify the $\mathrm{SU}(2)$ invariance in this function.

\subsection{Rotated MNIST}
We consider the classic example of image classification on the MNIST \cite{mnist} dataset. The original dataset is transformed by rotations by up to 45 degrees so that it has $\mathrm{SE}(2)$ symmetry which includes the rotations and translations on 2D grids. We set the number of generator channels to $c=1$. We set the search space to all affine transformations on 2D space, which is 6-dimensional. The discovery result is
$$
L=\begin{bmatrix}
    0.01 & -0.66 & 0.08 \\
    0.66 & -0.01 & -0.01 \\
    0 & 0 & 0
\end{bmatrix}
$$
This can be interpreted as a mixture of rotation ($L[1,0]=-L[0,1]=0.66$) and translation ($L[0,2]=0.08$), where the magnitude of rotations is larger than the magnitude of translations. Figure \ref{fig:rotmnist} visualizes the original and transformed MNIST digits.
\begin{figure}[h]
    \centering
    \includegraphics[width=.8\textwidth]{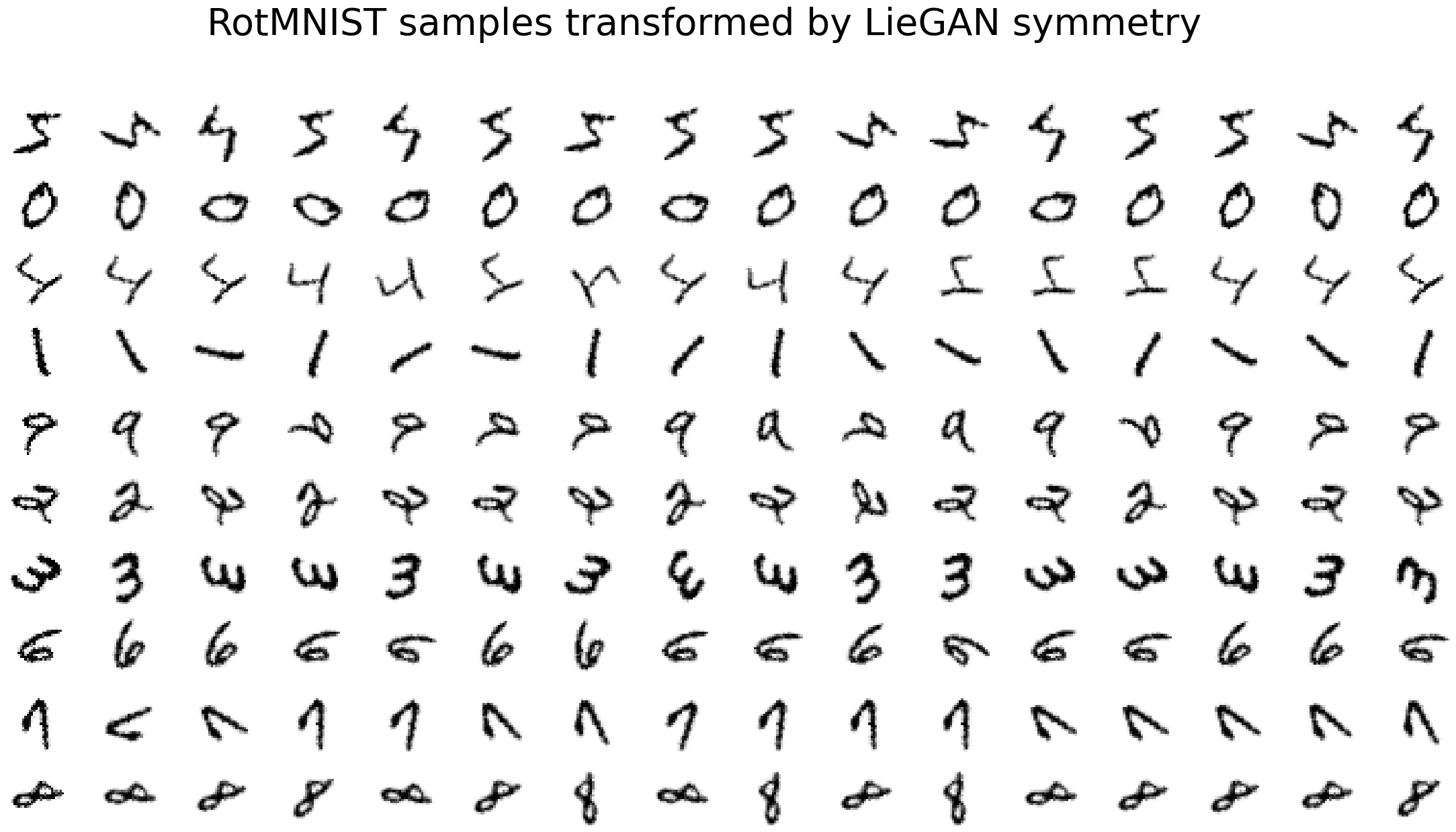}
    \caption{MNIST samples transformed by LieGAN. The first column shows the original samples from RotMNIST. For each image, we sample 15 group elements from LieGAN and plot the transformed images.}
    \label{fig:rotmnist}
\end{figure}

\subsection{Molecular Property Prediction}
We are also interested in whether LieGAN can recover $\mathrm{SE}(3)$ symmetry (the rotations and translations in 3D space), which has wide applications in computer vision, molecular dynamics, etc. Thus, we experiment on QM9 \cite{blum2009970, rupp2012fast}, where the task is to predict molecular properties based on the 3D coordinates and charges of the atoms.
\begin{figure}[h]
    \centering
    \includegraphics[width=.8\textwidth]{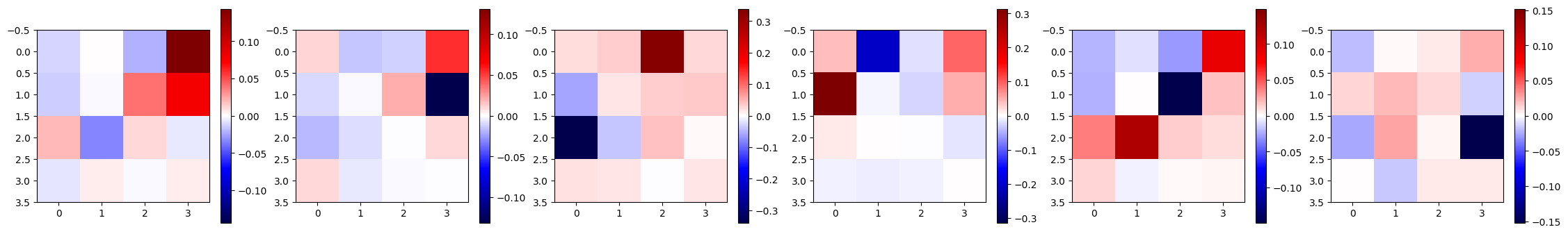}
    \caption{Discovery result for QM9 dataset.}
    \label{fig:qm9}
\end{figure}
We set the number of generator channels to $c=6$, which matches the dimension of $\mathrm{SE}(3)$. Figure \ref{fig:qm9} shows the discovered Lie algebra representations, which produce group representations that act on the affine coordinates $(x,y,z,1)$. LieGAN can discover an approximate $\mathrm{SE}(3)$ symmetry, where the skew-symmetric entries in the first three dimensions indicate rotations along different axes, and the non-zero entries in the last column indicate translations along different directions.

We can also use the discovered LieGAN symmetry to perform data augmentation during training. The discovered symmetry proves to increase the prediction accuracy on different QM9 tasks compared to a model with no symmetry, as is shown in table \ref{tab:qm9-pred}.
\begin{table}[h]
    \centering
    \begin{tabular}{c|ccc}
    \hline
    Task & No symmetry & LieGAN & $\mathrm{SE}(3)$ \\
    \hline
    HOMO & 52.7 & 43.5 & \textbf{36.5} \\
    LUMO & 43.5 & 36.4 & \textbf{29.8} \\
    \hline
    \end{tabular}
    \caption{Test MAE (in meV) on QM9 tasks. The results for no symmetry and $\mathrm{SE}(3)$ symmetry are referred from \citet{augerino}.}
    \label{tab:qm9-pred}
\end{table}


\end{document}